\documentclass{opt2026} 
\usepackage{caption}
\usepackage{subcaption}
\usepackage{amsfonts,amsmath,amssymb}
\usepackage{graphicx}
\usepackage{booktabs}
\usepackage{multirow}
\usepackage{tikz}
\usetikzlibrary{patterns}
\usepackage{pgfplots}
\pgfplotsset{compat=1.18}

\numberwithin{equation}{section}

\newcommand{\Xbar}{\bar X}
\newcommand{\xbar}{\bar x}
\newcommand{\Rz}{\mathbb{R}}
\newcommand{\Ex}{\mathbb{E}}
\newcommand{\one}{\mathbb{I}}

\makeatletter
\@ifundefined{remark}{}{}
\makeatother

\makeatletter
\renewcommand{\section}{\@startsection{section}{1}{\z@}%
   {-0.08in \@plus -1ex \@minus -.2ex}%
   {0.025in \@plus.2ex}%
   {\normalfont\rmfamily\bfseries\large\raggedright}%
}
\renewcommand\subsection{\@startsection{subsection}{2}{\z@}%
   {-0.07in \@plus -1ex \@minus -.2ex}%
   {0.02in \@plus .2ex}%
   {\normalfont\rmfamily\bfseries\normalsize\raggedright}%
}
\makeatother

\title[Escaping Local Optima in MAX-CUT] {Graph Coarsening Enables Escaping Local Optima in Compute-Constrained Nonconvex MAX-CUT}

\optauthor{%
\Name{M. Harper Langston} \Email{hlangsto@qti.qualcomm.com}\\
\Name{Pierre-David Letourneau} \Email{pletourn@qti.qualcomm.com}\\
\Name{Dalton Jones} \Email{daltjone@qti.qualcomm.com}\\
\Name{Harris Teague} \Email{hteague@qti.qualcomm.com}\\
\Name{Peter Scott} \Email{peter.scott@nyu.edu}\\
\Name{Mingu Lee} \Email{mingul@qti.qualcomm.com}\\
\Name{Richard A. Lethin} \Email{rlethin@qti.qualcomm.com}\\
\Name{Christopher Lott} \Email{clott@qti.qualcomm.com}\\
\Name{Roland Memisevic} \Email{rmemisev@qti.qualcomm.com}\\
\Name{Fatih Porikli} \Email{fporikli@qti.qualcomm.com}\\
\addr Qualcomm AI Research}

\begin{document}

\maketitle

\begin{abstract}
Because its binary
optimization landscape contains exponentially many local optima, MAX-CUT is highly sensitive to initialization. We
study Particle Polynomial Relaxation (PPR), a moment-based nonconvex continuous reformulation possessing an advantageous optimization landscape. We warm-start PPR with Combinatorial Multigrid (CMG), a
spectral graph-coarsening method that produces structure-aware initial
solutions in near-linear time without training data. On \texttt{t2g20} ($n=400$), a
benchmark instance with known exact optimum,
the combined CMG+PPR pipeline reaches $97.5\%$ of optimum on average
and $98.4\%$ on the best run.  An extended variant with discrete local-search refinement reaches the
published optimum on all sixty Biq Mac instances evaluated by a recent
neural branch-and-bound method. Using a sparse CSR representation, the
same approach scales to G-set instances with up to $20{,}000$ vertices,
reducing memory requirements by up to $3333\times$ while maintaining a
consistent improvement over the CMG warm start. Gains are largest
on sparse weighted graphs; on dense graphs we complement the pipeline with spectral sparsification.
\end{abstract}

\section{Introduction}
\label{sec:introduction}

MAX-CUT asks for a partition of a graph's vertices that maximizes the total
weight of crossing edges~\citep{karp1972reducibility,garey1979computers}. With $x_i\in\{-1,+1\}$, it can be written as the
nonconvex quadratic program
\begin{equation}
    \label{eq:max_cut_prob}
      \min_{x\in\{-1,+1\}^n} -\frac14\sum_{(i,j)\in E}w_{ij}(1-x_ix_j),
\end{equation}
equivalently as a ground-state search for the associated Ising
energy~\citep{barahona1982ising,lucas2014ising}. We first note that the flip neighborhood on $\{-1,+1\}^n$ may contain exponentially many local
optima; so, a local method's terminal cut can be determined more by the initial
point than by the amount of search performed afterward. 
Second, the classical
Goemans-Williamson $0.878$ approximation ratio ~\citep{goemans1995}, tight
under the Unique Games Conjecture in the worst case~\citep{khot2007}, does
not describe how close a nonconvex solver can get on structured instances when
initialization already captures useful graph information.

\begin{figure}[htbp!]
\centering
\includegraphics[width=\textwidth]{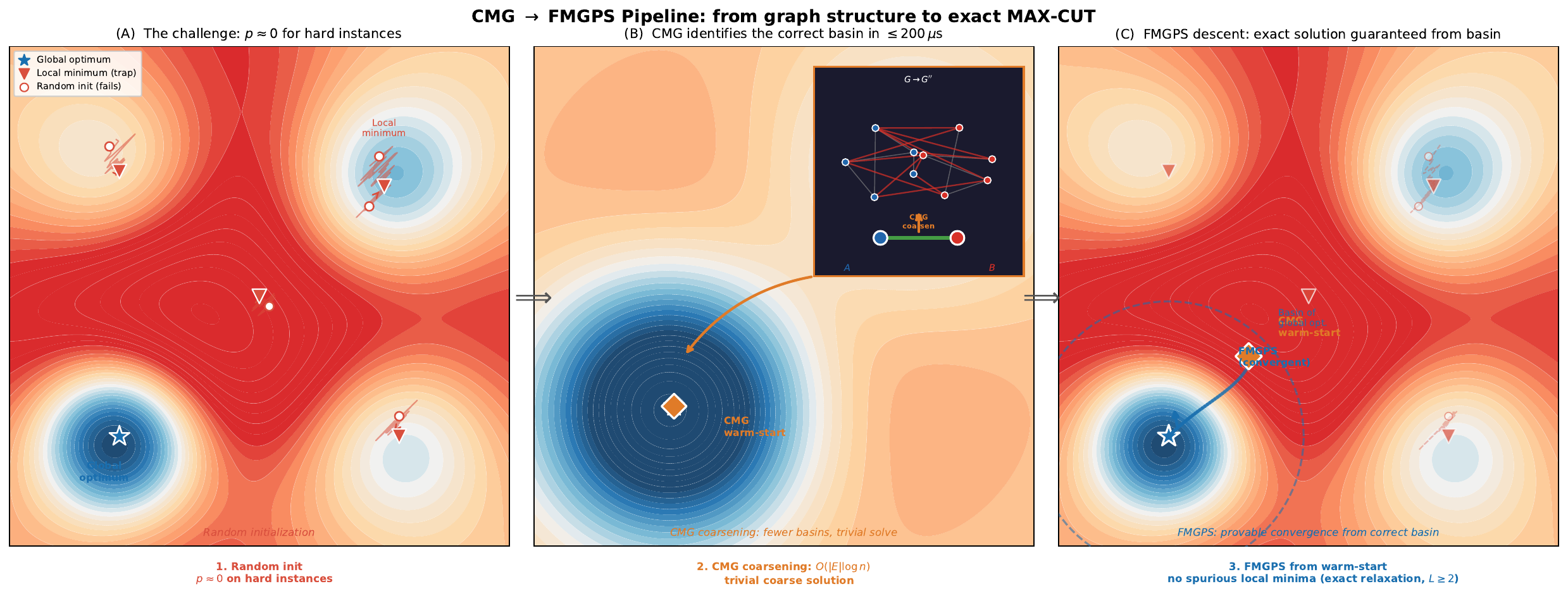}
\caption{\small Schematic comparison of random and structure-aware initialization.
  \textbf{Left:} random restarts are likely to terminate in different local
  optima of the flip landscape. \textbf{Center:} CMG coarsening summarizes the
  graph into larger structural units and returns a warm start. \textbf{Right:}
  PPR is then initialized in that basin and can make multi-node moves that are
  inaccessible to one-bit local search.}
\label{fig:naive_fails}
\end{figure}

We investigate whether two complementary ideas can be combined effectively. {\em First}: a continuous moment-based reformulation which possesses a highly-advantageous optimization landscape, namely, every suboptimal point has zero escape barrier when its mixing weight is free to move (Appendix~\ref{sec:escape_barrier_appendix} and~\citep{letourneau2023polyopt}). {\em Second}: a graph coarsening warm
start that places the finite-particle solver in a useful basin. We find that
this combination is effective on sparse and weighted structured instances and
that it scales to $n=20{,}000$ when the coupling matrix is stored sparsely.
The same experiments also identify a clear boundary: on dense or nearly
homogeneous graphs, CMG often has little structure to exploit, and the
pipeline provides little or no improvement over the warm start.  Appendix~\ref{app:extended} gives additional background, mechanical details,
and a small end-to-end walkthrough. The main text focuses on the solver, initialization, and empirical behavior of the composed pipeline.

\textbf{Relation to learned combinatorial solvers.} Recent learned MAX-CUT
methods supply structural information through training. OptGNN
\citep{yau2024optgnn} learns an overparameterized Burer-Monteiro
\citep{burer2003} update, while Chen et al.~\citep{chen2026neuralsdp} train a
GNN to predict SDP certificates used inside branch-and-bound. Our approach
uses a different source of structure: spectral graph coarsening on the input
graph itself. This requires no labels, no training distribution, and no neural
inference hardware. We do not claim uniform superiority, but on the six Biq
Mac families evaluated by~\citep{chen2026neuralsdp}, our full heuristic pipeline reaches the
same published optima on all $60$ instances without their training setup
(Table~\ref{tab:biqmac}).

\section{Method: moment relaxation with a graph-coarsening warm start}
\label{sec:method}

\subsection{Continuous Reformulation through Discrete/Particle Product Measures}
\label{sec:fmgps}

Our proposed PPR method is inspired by the work of Lasserre \citep{lasserre2001} and Letourneau et al.~\citep{letourneau2023polyopt} which both address the problem of global polynomial optimization through a moment-based reformulation. In this case, consider a \emph{level-$l$ discrete product measure} to be defined as,
\begin{equation}
  \mu^{(l)}(\cdot) \;=\; \bigotimes_{i=1}^D
    \Bigl(\sum_{r=1}^R w_i^{(l,r)}\,\delta_{\xbar_i^{(r)}}\Bigr),
  \qquad
  w_i^{(l,\cdot)}\in\Delta^{R-1},
\end{equation}
where $\Delta^{R-1}:=\{w\in\Rz_{\ge0}^R:\sum_rw_r=1\}$ is the
probability simplex; $R$ is thus the number of particles (point masses, or
atoms) used to represent each coordinate's marginal, and $l=1,\ldots,L$
indexes one of $L$ independently optimized such mixtures (\emph{levels}),
combined below through a mixing weight $\beta$. Its degree-$n$ marginal moment in coordinate $i$ is
\begin{equation}
  a_{n,i}^{(l)} \;:=\; \int T_n(x_i) \, \mathrm{d} \mu^{(l)}(x) = \sum_{r=1}^R w_i^{(l,r)}\,T_n\bigl(\xbar_i^{(r)}\bigr),
\end{equation}
where $\{T_n (x) \}$ is a one-dimensional graded polynomial basis. Given a polynomial \\$p(x) = \sum_{\mathbf{n} \in \mathrm{supp}(p)} p_{\mathbf{n}} \, T_{\mathbf{n}} (x)$,\footnote{$\mathbf{n} \in \mathbb{N}$ is a \emph{multi-index}. $\mathrm{supp}(p)$ is the set of nonzero coefficients of $p(\cdot)$ (in this specific basis) and $T_{\mathbf{n}}(x) := \prod_{i=1}^D T_{n_i}(x)$ is a multi-dimensional polynomial.} define the
\emph{single-level functional}
\begin{equation}
  g(w) \;:=\; \Ex_{\mu(w)}[p(x)]
       \;=\; \sum_{\mathbf n\in\mathrm{supp}(p)} p_{\mathbf n}
             \prod_{i=1}^D a_{n_i,i}(w_i),
  \qquad w\in(\Delta^{R-1})^D,
\end{equation}
and, for $L\ge1$ levels with a \emph{free} mixing weight
$\beta\in\Delta^{L-1}$, the objective
\begin{equation}
  F(w,\beta) \;:=\; \sum_{l=1}^L \beta_l\,g(w^{(l)})
  \qquad\text{on}\qquad
  \mathcal F \;:=\; (\Delta^{R-1})^{D\times L}\times\Delta^{L-1}.
\end{equation}
Call level $l$ \emph{present} if
$\beta_l>0$, and write $V_l:=g(w^{(l)})$ for its value. We prove in Appendix \ref{sec:escape_barrier_appendix} that the continuous optimization problem,\footnote{Our implementation parametrizes $w$ through softmax to enforce its simplex constraint with no fundamental differences to conclusions; $\beta$ is not parametrized or optimized at all, held fixed at the uniform average.}
\begin{equation}
    \min_{(w, \beta) \in \mathcal{F} } \,  F(w,\beta)
\end{equation}
is equivalent to the original MAX-CUT problem (Eq.~\ref{eq:max_cut_prob}); despite being nonconvex, every suboptimal point has zero escape barrier when $\beta$ is free to move\footnote{Can be strictly improved through a feasible path originating at this point}. In this case, the objective polynomial $p(x)$ is the max-cut objective of Eq.~\ref{eq:max_cut_prob}, and the particles' support is $\{-1,1\}$. Each $\mu_i^{(l)}$ is represented by $R$ atoms with softmax weights and $\tanh$-parameterized locations, making moment-matrix positivity guaranteed and leaving an unconstrained problem with $O(LnR)$ variables (linear in $n$ for fixed $L,R$, versus $O(n^2)$ for the GW SDP (Appendix~\ref{app:scaling})). We solve this problem over $w$ with L-BFGS~\citep{liu1989lbfgs} across a penalty-continuation schedule, holding $\beta$ fixed at the uniform average $\beta_l=1/L$ rather than optimizing it. \footnote{App.~\ref{app:fmgps_algorithm}, Prop.~\ref{prop:fixed_beta_bound} gives the weaker barrier bound this restriction leaves intact (details of the $w$-only solve in Alg.~\ref{alg:fmgps_maxcut})}



\subsection{Graph-coarsening initialization}
\label{sec:cmg}

At larger scales, finite-particle relaxation benefits from a strong initialization; we use
Combinatorial Multigrid (CMG)
~\citep{langston2019cmg,koutis2011combinatorial} to construct a
{\em structure-aware warm start}.  CMG builds a coarsen$\rightarrow$solve$\rightarrow$uncoarsen graph hierarchy:  at each level, vertex $i$ joins the cluster of
its strongest neighbor
$B[i]=\arg\max_{j\sim i} w_{ij}$ when
$M[i]/D[i]\geq1/8$, where
$M[i]=\max_{j\sim i}w_{ij}$ and
$D[i]=\sum_{j\sim i}w_{ij}$; otherwise it remains a singleton. We directly solve the
coarsest graph, and the solution is projected back to
the original graph with local-flip refinement at each level (total
cost $O(|E|\log n)$).

The resulting initialization is entirely graph-driven and requires no training data. Its value comes from exposing large-scale structure before continuous refinement. When edge weights are nearly uniform, however, little structure is available to exploit, the hierarchy compresses poorly, and the warm-start advantage diminishes. This behavior appears in the experiments and is consistent with the mechanism of the method.

\subsection{The composed pipeline}
\label{sec:pipeline}

CMG returns a binary warm start $x_{\mathrm{cmg}}$ that is already refined by
single-bit flips. From there, PPR searches over a
continuous particle representation that can make simultaneous multi-node
changes. In the stronger configuration, the best binary solution from pass
$p$ is used to initialize pass $p+1$. A single 2-opt~\citep{croes1958tsp} sweep is applied after
the final PPR pass (applying 2-opt after every
pass tends to lock the solution into a two-bit local basin too early
(Table~\ref{tab:t2g20})). For dense graphs with more than $3n$ edges, we first
apply the Spielman-Srivastava (S-S) effective-resistance sparsification
\citep{spielman2011}; this step is a no-op for the already-sparse structured
instances that dominate our results.

This is the pipeline we evaluate in Section~\ref{sec:results}: Spielman-Srivastava spectral
sparsification~\citep{spielman2011} (when needed) feeds a CMG warm start, which PPR refines on the
finest graph across one or more passes, closed by a single 2-opt sweep; an
optional discrete tail (Kernighan-Lin, 3-opt, restarts, and a memetic GA)
can be appended afterward. Figure~\ref{fig:pipeline_spine}
(Appendix~\ref{app:pipeline_end}) draws this end to end, and
Appendix~\ref{app:three_tools} summarizes how the three main stages
complement each other.

\section{Experiments}
\label{sec:results}

We evaluate our pipeline (Section~\ref{sec:pipeline} with details in Appendix~\ref{app:extended}) on structured toroidal Ising instances from the Biq Mac
library~\citep{wiegele2007biqmac} (branch-and-cut exact optima as ground
truth) and on real G-set instances up to $n=20{,}000$; ratios are
cut/optimum. We report $0.878$ as the poly.-time \emph{worst-case} guarantee~\citep{goemans1995} (not GW's typical achieved
cut, which is far higher on structured instances); where we do contrast
against an achieved GW cut we use a measured SDP-plus-rounding
solve.

\textbf{Structural warm start dominates raw search on the hard small
instance.} The $n{=}15$ instance shows that CMG finds the optimum in $30\,\mu$s, five orders of magnitude less time than the $21$ random PPR restarts consume, and the best of those restarts still reaches a cut of only $37$ (one short of the optimal $38$); none of the $21$ restarts finds the exact optimum. No additional restart budget closes the
gap. The instance is small, but it isolates the failure mode of random
initialization: once the solver enters the wrong basin, further local
optimization does not recover the optimum.

\textbf{Rank $R>1$ is necessary for the composed pipeline to improve on the
warm start.} Table~\ref{tab:particle_pipeline} reports the pipeline on two
toroidal families with published exact optima (of six tested; the other four
follow the same pattern), together with a dense expander contrast. The $R=1$
configuration barely moves off CMG; $R\geq2$ lifts every structured instance
$4$-$15$ points past CMG alone, well past the $0.878$ worst-case line. On
the expander \texttt{g05\_60}, by contrast, CMG is already at $0.96$ and
\emph{no} configuration improves it, as the dense-regime failure mode
described above predicts.

\begin{table}[htbp!]
\footnotesize\centering
\caption{\small S-S$\to$CMG$\to$PPR, mean cut/OPT ratio over three seeds vs. known optima~\citep{wiegele2007biqmac}; \texttt{g05\_60} here is a 5-instance subset of the family, vs.\ the full 10-instance mean in Table~\ref{tab:biqmac}. P-32,8 denotes $L=32,R=8$.}
\label{tab:particle_pipeline}
\setlength{\tabcolsep}{4pt}
\begin{tabular}{lrrrrc}\toprule
Family & $n$ & CMG & P-16,$R=1$ & \textbf{P-32,8} & Notes\\\midrule
\texttt{t3g5}  & 125 & 0.828 & 0.883 & \textbf{0.978} & 3-D toroidal\\
\texttt{t2g20} & 400 & 0.830 & 0.832 & \textbf{0.907} & 2-D toroidal\\\midrule
\multicolumn{4}{l}{\emph{Reference:}} & 0.878 & GW guarantee\\\midrule
\texttt{g05\_60} & 60 & 0.963 & 0.963 & 0.963 & Expander (5/10); CMG already strong\\
\bottomrule
\end{tabular}
\end{table}

\textbf{Pushed further, on the largest exact-optimum instance in this
comparison the pipeline
reaches $97.5\%$ mean and $98.4\%$ best.} \texttt{t2g20} ($n=400$) is the
largest instance in Table~\ref{tab:particle_pipeline} with a published exact
optimum, and
the weakest single-shot result above. Chaining five PPR passes lifts CMG's
$0.832$ to $0.952$; a single closing 2-opt sweep lifts this to
$\mathbf{0.975}$ mean, $\mathbf{0.984}$ best seed (Table~\ref{tab:t2g20} and
Figure~\ref{fig:t2g20_progression}, both in
Appendix~\ref{app:additional_results}).
Interleaving 2-opt after \emph{every} pass instead locks the solution into a
2-opt basin after pass~1, finishing at only $0.957$, so the ordering (search, then polish once) matters. Restart count saturates quickly
($k{=}40\to80\to160$ buys $+0.003$ then $+0.001$), confirming that the
quality of the warm-start basin, not restart breadth, is the bottleneck
(Figure~\ref{fig:pass_conv}).

\textbf{With a discrete polish tail, our pipeline reaches the published
optimum on all sixty Biq Mac instances.}
Chen et al.~\citep{chen2026neuralsdp} train a GNN to accelerate B\&B to a \emph{proven} optimum on $6$ Biq Mac families ($60$
instances). Our pipeline with a discrete tail (multi-pass PPR $\to$
2-opt~\citep{croes1958tsp} $\to$ Kernighan-Lin~\citep{kernighan1970efficient}
$\to$ 3-opt~\citep{lin1973heuristic} $\to$ $200$-restart $\to$ memetic
GA~\citep{moscato2004memetic,chu1998genetic}) matches the published optimum on \emph{every
one of the $60$} (Table~\ref{tab:biqmac}, Figure~\ref{fig:biqmac_neural}) in
$8.5$-$130.5$s per instance, against B\&B times up to
$40{,}930$s (Mosek SDP bound) or $4{,}509$s (learned proxy). While they prove optimality, we report a heuristic cut that
matches a published one.  However, it answers the comparison the learned solvers set up, with no training data or
inference hardware. The \texttt{g05} expanders also confirm
Table~\ref{tab:particle_pipeline}'s dense-regime pattern: warm start alone
already sits at $0.97$-$0.98$, so light PPR closes most of the
gap.

\begin{table}[htbp!]
\footnotesize\centering
\caption{\small $10$ instances chosen from the $6$ Biq Mac families~\citep{chen2026neuralsdp}; full pipeline matches all $60$  published optimum. Mean single-core per-instance wall clock used for timing comparisons; B\&B
  times from~\citep{chen2026neuralsdp}, Tab. $3$.}
\label{tab:biqmac}
\setlength{\tabcolsep}{5pt}
\begin{tabular}{lr rr r rr}\toprule
& & \multicolumn{3}{c}{Ours (heuristic)} & \multicolumn{2}{c}{\citep{chen2026neuralsdp} B\&B (s)}\\
\cmidrule(lr){3-5}\cmidrule(lr){6-7}
Family & $n$ & CMG & Pipeline & time\,(s) & Vanilla (s) & Neural (s)\\\midrule
\texttt{g05\_60}   & 60  & 0.975 & \textbf{1.000} & 23.9  & 53    & 9.5\\
\texttt{g05\_80}   & 80  & 0.977 & \textbf{1.000} & 63.8  & 1170  & 156\\
\texttt{g05\_100}  & 100 & 0.982 & \textbf{1.000} & 130.5 & 40930 & 4509\\
\texttt{pm1s\_80}  & 80  & 0.817 & \textbf{1.000} & 8.5   & 610   & 81\\
\texttt{pm1s\_100} & 100 & 0.818 & \textbf{1.000} & 13.7  & 10382 & 1058\\
\texttt{w01\_100}  & 100 & 0.881 & \textbf{1.000} & 12.7  & 11349 & 1074\\
\bottomrule
\end{tabular}
\end{table}

\textbf{At the continuous phase, L-BFGS beats Adam;
but, the discrete tail sets the final cut, not the optimizer.} PPR solves each
penalty-continuation stage with a single L-BFGS call; we swapped that inner
solve for Adam~\citep{kingma2014adam}, full-batch gradient descent (GD), and
a level-subsampled stochastic variant based on stochastic
approximation~\citep{robbins1951stochastic} on eight instances, with the warm
start, restarts, continuation schedule, and downstream discrete tail held
fixed (full per-instance table in Appendix~\ref{app:optimizer}). Pre-tail,
the four disagree by up to $7.4\%$\footnote{L-BFGS's curvature edge leaves
Adam up to $6.2\%$ behind, and GD, with neither momentum nor curvature,
stalls at the CMG warm start on \texttt{G1} and is $11$-$30\times$ slower
per instance.}, but the identical discrete tail collapses that spread to
$\leq0.7\%$ on every instance: L-BFGS just gets there the cheapest.

\textbf{The warm start's advantage persists to n=20{,}000.}
A dense coupling matrix requires $8n^2$ bytes ($3.2$\,GB at $n=20{,}000$), impractical for sparse-signed graphs ($|E|\!\approx\!2n$). CSR
representations storing only $2|E|$ nonzeros resolve this. Running
CMG $\to$ light PPR $\to$ 2-opt on $6$ real G-set instances,
$n=5K$--$20K$ (Table~\ref{tab:sparse}), reduces memory by
$833$--$3333\times$ with linear scaling in $n$. Gains over CMG warm
start through the continuous-relaxation path alone,
$\Delta_{\mathrm{2opt}}$, remain in a tight $1.19$--$1.21\times$ band,
matching the behavior at $n=125$--$1000$
(Figure~\ref{fig:sparse_scaling}). Adding a discrete
tail\footnote{Kernighan-Lin, 3-opt, a memetic GA, and an iterated local
search: the retuned configuration of Table~\ref{tab:g81_stages}.}
raises the final gain to $\Delta_{\mathrm{final}}=1.32$--$1.33\times$,
again with no observable decay. For \texttt{G77} and \texttt{G81}, the full pipeline reaches $\mathbf{98.7\%}$ and
$\mathbf{99.0\%}$ (the continuous-relaxation path alone
reaches $90.8\%$ and $89.8\%$) of the externally verified reported optima (Cosm's~\citep{zick2026cosm} $9940$ and $14060$). A tabu-diversified
tail~\citep{glover1986tabu,glover1997tabu} reaches $13962$
($99.3\%$ of $14060$) on \texttt{G81} in $\sim$213\,s and provides the
study's only internal optimality certificate, verified in a further
$\sim$2.4\,s across $40$ connected cores
(Appendix~\ref{app:g81cert}).

\begin{table}[htbp!]
\footnotesize\centering
\caption{\small Sparse CSR path, four of six G-set instances, single core.
  ``mem'': dense/CSR storage ratio. ``2opt'': cut after CMG$\to$PPR$\to$2-opt
  only, with its own gain over CMG, $\Delta_{\mathrm{2opt}}$ (this is the
  configuration the pipeline name describes). ``final'': cut after also
  appending the retuned discrete tail of Table~\ref{tab:g81_stages}
  (KL/3-opt/GA/ILS), with $\Delta_{\mathrm{final}}$. ratio-to-OPT given at
  both stages, only where an externally verified optimum exists.}
\label{tab:sparse}
\setlength{\tabcolsep}{3pt}
\begin{tabular}{lrr rrr rrr}\toprule
Instance & $n$ & mem & cmg & 2opt & $\Delta_{\mathrm{2opt}}$ & final & $\Delta_{\mathrm{final}}$ & ratio-to-OPT (2opt/final)\\\midrule
\texttt{G57} & 5000  & 833.3$\times$  & 2600  & 3142  & 1.208 & 3460  & 1.331 & -- \\
\texttt{G67} & 10000 & 1666.7$\times$ & 5198  & 6258  & 1.204 & 6872  & 1.322 & -- \\
\texttt{G77} & 14000 & 2333.3$\times$ & 7454  & 9022  & 1.211 & 9810  & 1.316 & 0.908 / \textbf{0.9869} \\
\texttt{G81} & 20000 & 3333.3$\times$ & 10574 & 12626 & 1.194 & 13922 & 1.317 & 0.898 / \textbf{0.9902} \\
\bottomrule
\end{tabular}
\end{table}

\section{Discussion: comparison and limitations}
\label{sec:comparison}

\textbf{Comparable assumptions to bespoke heuristics.}
Cosm~\citep{zick2026cosm} reports optimality on \texttt{G77} and
\texttt{G81} using an external exact solver rather than an internal proof.
Our ratio-to-optimum results rely on the same type of reference. We do not
aim to match Cosm's runtime performance; our contribution is a
general framework that combines reusable structure-aware components (nonconvex relaxation and spectral coarsening) with a broader class of
constraints rather than a solver tailored to MAX-CUT.

\textbf{When the method provides little benefit.}
On dense, well-conditioned graphs, CMG coarsening succeeds almost everywhere, so the warm
start is already near the local-flip limit and PPR yields no measurable
improvement (Table~\ref{tab:particle_pipeline}). On uniform-weight
Erd\H{o}s-R\'enyi graphs, the CMG grouping criterion
$M[i]/D[i]\geq1/8$ rarely holds once $n>16$, preventing meaningful
coarsening. In both cases, the pipeline exploits
weight heterogeneity, so these graph families provide little structure to
exploit. This behavior is consistent with our design goal of a
{\em structure-aware method}.

\textbf{Open questions.}
A discrete polishing stage (Kernighan--Lin, 3-opt, or a memetic GA) closes
most of the remaining gap on the hardest instances. However, a
random-start baseline with the same polishing budget matches or exceeds the
CMG-seeded pipeline on roughly half of these cases. Once sufficient
restarts are available, final solution quality is often determined more by
the polishing stage than by the warm start
(Table~\ref{tab:tail_ablation}; see also
Table~\ref{tab:t2g20_tail}). Understanding this trade-off and extending the
approach beyond $n=20{,}000$ under fixed compute remain promising
directions.

\section{Conclusion}
A structure-aware CMG warm start combined with a particle-based continuous
relaxation achieves near-optimal MAX-CUT solutions across a wide range of
problem sizes. The approach reaches $97.5\%$ mean (best-seed $98.4\%$) of
the exact optimum on \texttt{t2g20} ($n=400$), and $98.7$--$99.0\%$ of
the exact optimum on the two largest G-set instances with externally
verified optima. These gains persist up to $20{,}000$ vertices without
observable degradation, although they naturally diminish on graph families
with little exploitable structure. Appendix~\ref{sec:escape_barrier_appendix} discusses the theoretical motivation that continuous relaxation eliminates escape barriers under free mixing, while
the CMG warm start escapes the limitations of purely local search. In compute-constrained
regimes, this combination provides a robust advantage; however, when large
discrete-search budgets become available, the polishing stage increasingly
determines final solution quality. Consistent with the tail-ablation study
(Table~\ref{tab:tail_ablation}), a random start carried through the same
high-budget discrete tail can match or exceed the CMG-seeded pipeline on
roughly half of the hardest instances tested.

\bibliography{biblio}

\newpage
\appendix

\section{Particle Polynomial Relaxation (PPR) for MAX-CUT: full algorithm}
\label{app:fmgps_algorithm}

Algorithm~\ref{alg:fmgps_maxcut} specializes the particle-parameterization
solver of~\citet{letourneau2023polyopt} to the MAX-CUT setting. The only
constraint is the binary/box penalty
$\tfrac{\rho}{2}\|\mathrm{mean}_l(a2_l)-\mathbf1\|^2$ appearing in the
objective below, which drives each marginal's second moment to one and is
therefore zero exactly when every marginal is supported on $\{-1,+1\}$. There is
no equality-constraint dual loop; instead, $K$ independent restarts are run
through a fixed penalty-continuation schedule $\rho_1<\cdots<\rho_S$.  Pseudocode for the rest of the pipeline is presented in Appendix~\ref{app:component_algorithms}.

\begin{algorithm2e}[htbp!]
\small
\SetAlgoLined
\DontPrintSemicolon
\caption{\small PPR for MAX-CUT}
\label{alg:fmgps_maxcut}
\KwIn{$J$ ($n\times n$, normalized); warm start $x_0$; $L,R,k$; sat;
  penalty schedule $\rho_1<\cdots<\rho_S$}
\KwOut{best $\pm1$ cut found}
$\mathrm{best}\gets-\infty$\;
\For{restart $=1,\ldots,k$}{
  $u_{l,0,i} \gets \operatorname{atanh}(\mathrm{sat}\cdot x_{0,i}) + \mathcal{N}(0,0.02)$ for all $l$
    \tcp*{warm-started slot}
  all other $u_{l,r,i}\sim\mathcal{N}(0,0.02)$ \tcp*{unwarmed particles}
  \For{$\rho = \rho_1,\ldots,\rho_S$}{
    form marginals $a_l = \sum_r w_r\, \tanh(u_{l,r,\cdot})$,
      $a2_l = \sum_r w_r\, \tanh(u_{l,r,\cdot})^2$\;
    $F(u) \gets -\dfrac{1}{2L}\sum_l a_l^\top J a_l
      + \dfrac{\rho}{2}\big\|\mathrm{mean}_l(a2_l) - \mathbf{1}\big\|^2$\;
    run L-BFGS on $F(u)$ to convergence\;
  }
  \For{each $(l,r)$}{
    $\hat{x}_i \gets \mathrm{sign}(\tanh(u_{l,r,i}))$ for all $i$\;
    local-flip refine $\hat{x}$ on $G$\;
    keep $\hat{x}$ if $\mathrm{Cut}(\hat{x}) > \mathrm{best}$\;
  }
}
\Return best $\hat{x}$ found across all restarts and $(l,r)$\;
\end{algorithm2e}

\subsection{Theory and Proofs}
\label{sec:escape_barrier_appendix}
In this appendix, we demonstrate that the proposed continuous reformulation for the MAX-CUT problem possesses an advantageous optimization landscape. Namely, we show that for every non-optimal feasible point, there exists a path of descent (not necessarily strict). In other words, the reformulated objective does not possess any strict local minima other then global minima (which may or may not be strict either). The arguments are processed without the softmax parametrization for simplicity. It can be shown that the results still hold following such parametrization.

\paragraph{Weights, levels, and the objective}
Recall that \emph{level-$l$ discrete product measure} is
\begin{equation}
  \label{eq:measure}
  \mu^{(l)}(\cdot) \;=\; \bigotimes_{i=1}^D
    \Bigl(\sum_{r=1}^R w_i^{(l,r)}\,\delta_{\xbar_i^{(r)}}\Bigr),
  \qquad
  w_i^{(l,\cdot)}\in\Delta^{R-1},
\end{equation}
where $\Delta^{R-1}:=\{w\in\Rz_{\ge0}^R:\sum_rw_r=1\}$ is the
probability simplex. Its degree-$n$ marginal moment in coordinate $i$ is
\begin{equation}
  \label{eq:moment}
  a_{n,i}^{(l)} \;:=\; \sum_{r=1}^R w_i^{(l,r)}\,T_n\bigl(\xbar_i^{(r)}\bigr),
\end{equation}
so $a_{0,i}^{(l)}\equiv1$ is pinned by the simplex constraint. Define the
\emph{single-level functional}
\begin{equation}
  \label{eq:g}
  g(w) \;:=\; \Ex_{\mu(w)}[p(x)]
       \;=\; \sum_{\mathbf n\in\mathrm{supp}(p)} p_{\mathbf n}
             \prod_{i=1}^D a_{n_i,i}(w_i),
  \qquad w\in(\Delta^{R-1})^D,
\end{equation}
and, for $L\ge1$ levels with a \emph{free} mixing weight
$\beta\in\Delta^{L-1}$, the objective
\begin{equation}
  \label{eq:F}
  F(w,\beta) \;:=\; \sum_{l=1}^L \beta_l\,g(w^{(l)})
  \qquad\text{on}\qquad
  \mathcal F \;:=\; (\Delta^{R-1})^{D\times L}\times\Delta^{L-1}.
\end{equation}
Call level $l$ \emph{present} if
$\beta_l>0$, and write $V_l:=g(w^{(l)})$ for its value.


\paragraph{Oscillation}
Let
\begin{equation}
  \label{eq:oscillation}
  C \;:=\; \max_{\xbar\in\Xbar}p(\xbar) - \min_{\xbar\in\Xbar}p(\xbar)
     \;=\; \max_{\xbar\in\Xbar}p(\xbar) - p^*
\end{equation}
be the (finite) oscillation of $p$ over the grid. Since $g(w)$
in~\eqref{eq:g} is an expectation of $p$ under a probability measure
supported on $\Xbar$, every feasible $w$ satisfies
$p^*\le g(w)\le p^*+C$, hence
\begin{equation}
  \label{eq:oscillation_bound}
  |g(w)-g(w')| \;\le\; C
  \qquad\text{for all feasible } w,w'.
\end{equation}

\begin{definition}[Escape barrier]
\label{def:escape_barrier}
For $L\ge1$ let $F$ and $\mathcal F$ be as in~\eqref{eq:F}. The
\emph{escape barrier} at a feasible point $(w,\beta)$ is
\begin{equation}
  \label{eq:escape_barrier}
  B(w,\beta) := \inf_{\pi}\ \Bigl[\ \max_{t\in[0,1]}F(\pi(t))
    \;-\; F(\pi(0))\ \Bigr],
\end{equation}
the infimum taken over all continuous feasible paths
$\pi:[0,1]\to\mathcal F$ with $\pi(0)=(w,\beta)$ and
$F(\pi(1))<F(\pi(0))$ (so $B:=+\infty$ when no such path exists). Note
$B\ge0$ always, since $t=0$ is admissible in the maximum.
\end{definition}


The barrier dichotomy needs one elementary fact in both of its
directions, so we record it first.

\begin{theorem}[Value equivalence]
\label{thm:equivalence}
$\displaystyle \pi^* := \min_{(w,\beta)\in\mathcal F} F(w,\beta) = p^*$.
\end{theorem}
\begin{proof}
\noindent\textbf{Step 1 ($\pi^*\le p^*$).}
Let $\xbar^*\in\Xbar$ be a global minimizer of $p$ over the grid (which
exists since $\Xbar$ is finite). For each $i$ let $r_i^*$ be the index
with $\xbar_i^{(r_i^*)}=\xbar_i^*$, and set
$w_i^{(1,r)}:=\one[r=r_i^*]$ (one-hot), with $\beta_1=1$. This is
feasible, and $F$ evaluates exactly to $p(\xbar^*)=p^*$ term by term.
Hence $\pi^*\le p^*$.

\noindent\textbf{Step 2 ($\pi^*\ge p^*$).}
For \emph{any} feasible $w$, each level's contribution
$\Ex_{\mu^{(l)}(w)}[p]$ is, by definition~\eqref{eq:measure}, the
expectation of $p$ under a genuine probability measure supported on
$\Xbar$, a convex combination of point masses at grid points, since
$\mu^{(l)}$ is a product of finitely supported marginals. Every such
expectation is $\ge\min_{\xbar\in\Xbar}p(\xbar)=p^*$. Any convex
combination of these expectations\footnote{In particular, a free
$\beta\in\Delta^{L-1}$, not merely a uniform average} preserves the
bound, so $F(w,\beta)\ge p^*$ for every feasible $(w,\beta)$, giving
$\pi^*\ge p^*$.

Combining, $\pi^*=p^*$.
\end{proof}

\begin{theorem}[Lateral-escape: zero barrier at every suboptimal
  point]
\label{prop:simplex_descent}
Let $L\ge2$ and let $(w,\beta)\in\mathcal F$ satisfy $F(w,\beta)>p^*$.
Then for every $\delta\in(0,1]$ there is a continuous feasible path
$\pi:[0,1]\to\mathcal F$ with $\pi(0)=(w,\beta)$ such that
\begin{equation}
  \label{eq:lateral_bound}
  F(\pi(1)) = p^*
  \qquad\text{and}\qquad
  \max_{t\in[0,1]}F(\pi(t)) \;\le\; F(w,\beta) + \delta\,C .
\end{equation}
Consequently $B(w,\beta)=0$: the net descent $F(w,\beta)-p^*>0$ is a
fixed positive quantity independent of $\delta$, while the barrier that
must be crossed to realize it is at most $\delta C$ and hence
arbitrarily small. Conversely, if $F(w,\beta)=p^*$ then $F\ge p^*$
everywhere on $\mathcal F$ (Step~2 of Theorem~\ref{thm:equivalence})
rules out any feasible path reaching a lower value, so
$B(w,\beta)=+\infty$. Together the two directions give the full
dichotomy: $B(w,\beta)=0$ iff $F(w,\beta)>p^*$, and $B(w,\beta)=+\infty$
if and only if $F(w,\beta)=p^*$, with no intermediate case (no suboptimal point
is protected by a barrier of positive height.
\end{theorem}
\begin{proof}
Let $\xbar^*\in\Xbar$ be a global minimizer of $p$ over the grid and let
$e(\xbar^*)\in(\Delta^{R-1})^D$ be the corresponding one-hot weight
vector, so $g(e(\xbar^*))=p(\xbar^*)=p^*$ (this is the only place a
specific weight vector is needed, and it requires nothing of $R$ beyond
$\xbar^*\in\Xbar$; no Vandermonde or rank hypothesis enters). Write
$V_l:=g(w^{(l)})$ and call level $l$ \emph{present} if $\beta_l>0$. Two
cases.

\smallskip
\noindent\textbf{Case 1: some level $l_1$ is absent
($\beta_{l_1}=0$).}
\emph{Phase A (exactly flat).} Holding $\beta$ and every other level
fixed, move $w^{(l_1)}$ to $e(\xbar^*)$ along the straight line
$(1-t)w^{(l_1)}+t\,e(\xbar^*)$, feasible for all $t$ by convexity of
each $\Delta^{R-1}$. Since $F$ depends on $w^{(l_1)}$ only through the
product $\beta_{l_1}g(w^{(l_1)})$ and $\beta_{l_1}=0$, $F$ is
\emph{constant} along the whole of Phase~A: the rise is exactly $0$.
\emph{Phase B (exactly affine, strictly decreasing).} Now move
$\beta_{l_1}$ from $0$ to $1$, taking the removed mass from the other
levels proportionally, i.e.\ $\beta_{l_1}(t)=t$ and
$\beta_l(t)=(1-t)\beta_l$ for $l\neq l_1$. This stays in $\Delta^{L-1}$,
and since every $V_l$ is now fixed,
$F(t)=(1-t)F(w,\beta)+t\,p^*$ is exactly affine with constant slope
$p^*-F(w,\beta)<0$: strictly decreasing, ending at $p^*$. The path never
rises at all, so~\eqref{eq:lateral_bound} holds with room to spare, for
every $\delta$.

\smallskip
\noindent\textbf{Case 2: every level is present ($\beta_l>0$ for all
$l$).} Fix $\delta\in(0,1]$, choose $l_1\in\arg\max_l V_l$ (so
$V_{l_1}\ge\sum_l\beta_lV_l=F(w,\beta)$), and set
$\delta':=\min(\delta,\beta_{l_1})>0$.

\emph{Phase A (lateral: non-increasing, and exactly flat at a critical
point).} Holding all $w^{(l)}$ fixed, decrease $\beta_{l_1}$ from
$\beta_{l_1}$ to $\delta'$, redistributing the removed mass
$\beta_{l_1}-\delta'\ge0$ proportionally among the other levels:
$\beta_l(t)$ interpolates affinely, remaining in $\Delta^{L-1}$
throughout. Mass is thereby moved \emph{out of} a level of maximal value
$V_{l_1}$ and \emph{into} a convex combination of levels each of value
$\le V_{l_1}$, so $F$ is non-increasing along Phase~A; the rise is
$\le0$.\footnote{When all present levels share a common value, exactly the
situation forced at any critical point by Step~2 of
Theorem~\ref{thm:equivalence}, Phase~A is
\emph{identically flat}: it is a genuine lateral move, neither ascending
nor descending, and it is available regardless of how deeply each
individual level is trapped.}

\emph{Phase B (the only ascent, bounded by $\delta'C$).} With $\beta$
now fixed at $\beta_{l_1}=\delta'$, move $w^{(l_1)}$ to $e(\xbar^*)$
along the straight line as in Case~1. Only the term
$\beta_{l_1}g(w^{(l_1)})=\delta'g(w^{(l_1)})$ varies, so for every $t$
\begin{equation}
  F(\pi(t)) - F(\pi(t_A))
    = \delta'\bigl[g(w^{(l_1)}(t)) - V_{l_1}\bigr]
    \;\le\; \delta' C \;\le\; \delta C ,
\end{equation}
where $t_A$ is the end of Phase~A and the bound uses
$|g(w)-g(w')|\le C$ from~\eqref{eq:oscillation_bound}. Since Phase~A did
not raise $F$, the running maximum over Phases~A-B is at most
$F(w,\beta)+\delta C$.

\emph{Phase C (exactly affine, strictly decreasing).} Now $V_{l_1}=p^*$
and all $V_l$ are again fixed; move $\beta_{l_1}$ from $\delta'$ to $1$,
taking mass from the others proportionally as in Case~1's Phase~B. $F$
is exactly affine in the interpolation parameter with endpoint value
$p^*$, hence non-increasing (strictly decreasing unless the starting
value already equals $p^*$, in which case the endpoint is reached
immediately), and $F(\pi(1))=p^*$. Phase~C therefore adds no new
maximum, and~\eqref{eq:lateral_bound} holds.

\smallskip
In both cases the path satisfies~\eqref{eq:lateral_bound}, and
$F(\pi(1))=p^*<F(w,\beta)$ makes $\pi$ admissible in
Definition~\ref{def:escape_barrier}; letting $\delta\downarrow0$ gives
$B(w,\beta)=0$. The converse direction stated above needs no further
argument beyond $F\ge p^*$ everywhere.
\end{proof}

The zero-barrier statement is deliberately \emph{weaker} than ``local
minimum $\Rightarrow$ global minimum,'' and the reason is not an
artifact of the fixed grid. The stronger statement would require a path
whose derivative is strictly negative \emph{immediately} at $t=0$, hence
a mechanism for spawning fresh mass at $\xbar^*$ inside a level that
already carries mass without first emptying it. Here the level to be
repositioned is first emptied to mass $\delta'$ (otherwise
repositioning it would cost $O(1)$ rather than $O(\delta)$), and the
emptying is lateral, not descending. The result is a path that descends
overall but not from the outset, which is exactly the distinction
between $B=0$ (proved here) and ``not a local minimum'' (false here). The operative consequence
survives: $p^*$ is reachable from anywhere at arbitrarily small cost, so
a solver willing to accept a transient increase of size $\delta C$ is
never blocked by a genuine energy barrier, only by the difficulty of
finding the lateral direction.

\paragraph{What this buys the deployed algorithm.} Theorem
\ref{prop:simplex_descent}'s mechanism needs $\beta$ to move: Case~1 needs an
\emph{absent} level ($\beta_{l_1}=0$) and Case~2 needs $\beta_{l_1}$
\emph{shrunk} to a small $\delta'$ before the repositioning move. Algorithm
\ref{alg:fmgps_maxcut}, and every solver run in this paper, never varies
$\beta$ at all: it always averages all $L$ levels uniformly,
$\beta_l=1/L$. Eq.~\eqref{eq:F} is stated for a general, possibly free
$\beta\in\Delta^{L-1}$ because the construction and Theorem
\ref{thm:equivalence} hold at that generality (Step~2's bound is proved for
\emph{any} convex combination, uniform included), but Theorem
\ref{prop:simplex_descent}'s zero barrier is specific to a $\beta$ that is
free to move, and so does not, by itself, describe what Algorithm
\ref{alg:fmgps_maxcut} is guaranteed. The following proposition states what
survives when $\beta$ is pinned at the uniform average.

\begin{proposition}[Barrier bound under the fixed uniform weighting]
\label{prop:fixed_beta_bound}
Let $\beta_l=1/L$ for every $l$, held fixed, and let $w$ be feasible with
$F(w,\tfrac1L\mathbf1)>p^*$. Define the \emph{fixed-$\beta$ escape barrier}
$B_{\mathrm{fix}}(w)$ exactly as in Definition~\ref{def:escape_barrier}, but
with the infimum restricted to continuous paths that vary only $w$ (as
Algorithm~\ref{alg:fmgps_maxcut} does; $\beta$ is not a free coordinate of
its search). Then
\[
  B_{\mathrm{fix}}(w) \;\le\; C/L .
\]
\end{proposition}
\begin{proof}
Write $V_l:=g(w^{(l)})$ as before and choose $l_1\in\arg\max_lV_l$, so
$V_{l_1}\ge\sum_l\tfrac1LV_l=F(w,\tfrac1L\mathbf1)$, exactly as in
Theorem~\ref{prop:simplex_descent}'s Case~2. Move $w^{(l_1)}$ along the
straight line $(1-t)w^{(l_1)}+t\,e(\xbar^*)$ to the global minimizer's one-hot
weight, holding $\beta$ (hence $\beta_{l_1}=1/L$) and every other level
fixed throughout; there is no Phase~A here, since redistributing $\beta$
is unavailable. Only the term $\tfrac1Lg(w^{(l_1)})$ varies, so for every $t$
\[
  F(\pi(t)) - F(w,\tfrac1L\mathbf1)
    = \tfrac1L\bigl[g(w^{(l_1)}(t)) - V_{l_1}\bigr]
    \;\le\; \tfrac1LC ,
\]
using $g(w^{(l_1)}(t))\le p^*+C$ and $V_{l_1}\ge p^*$
(Eq.~\eqref{eq:oscillation_bound}), so the running maximum over the whole
path is at most $F(w,\tfrac1L\mathbf1)+C/L$. At $t=1$,
$F(\pi(1))-F(w,\tfrac1L\mathbf1)=\tfrac1L(p^*-V_{l_1})<0$, a strict decrease
since $V_{l_1}\ge F(w,\tfrac1L\mathbf1)>p^*$. Hence $\pi$ is admissible in
the (restricted) Definition~\ref{def:escape_barrier}, giving
$B_{\mathrm{fix}}(w)\le C/L$.
\end{proof}
\noindent As $L\to\infty$ this recovers Theorem
\ref{prop:simplex_descent}'s $B=0$: it is exactly the freedom to shrink
$\beta_{l_1}\to0$ in Cases~1-2 that pinning $\beta$ at $1/L$
forecloses, leaving $\delta'=1/L$ (not driveable $\to0$ at fixed $L$) as
the best available discount on the oscillation bound. Algorithm
\ref{alg:fmgps_maxcut} is therefore guaranteed only the weaker $C/L$ barrier
of Proposition~\ref{prop:fixed_beta_bound}, not the zero barrier of Theorem
\ref{prop:simplex_descent}; realizing the stronger guarantee would require
implementing a solver that actually moves $\beta$, which this paper does not
do.

\section{Extended background, related work, pipeline mechanics, further results and contributions}
\label{app:extended}

This appendix records additional material compressed in the main text:
related approaches, the need for structured initialization, the roles of the
three pipeline stages, the binary-constraint derivation, and the mechanics of
coarsening and refinement. It is included for completeness and does not change
any claims in the main text.

\subsection{Existing approaches}
\label{app:existing_approaches}

MAX-CUT algorithms fall into several broad categories. Exact methods such as
branch-and-bound (B\&B) and dynamic programming provide certificates but scale
exponentially in the worst case. The Biq Mac library~\citep{wiegele2007biqmac}
uses exact methods to provide optima for the structured instances used as
ground truth here. The Lasserre hierarchy~\citep{lasserre2001} gives a
sequence of increasingly tight semidefinite relaxations, but its moment
matrices grow quickly with $n$. Goemans-Williamson~\citep{goemans1995} solves
an SDP relaxation and rounds by random hyperplanes, giving the classical
$0.878$ worst-case approximation ratio. Learned solvers form a more recent
category: OptGNN~\citep{yau2024optgnn} learns updates in a Burer-Monteiro
relaxation, Chen et al.~\citep{chen2026neuralsdp} learn SDP certificates for
branch-and-bound, and Dong et al.~\citep{dong2025streaming} study streaming
algorithms with an oracle-like learned labeling. Our method supplies the
structural prior through graph coarsening instead of learning.

\subsection{Why structured initialization is necessary}
\label{app:challenge}

The finite-particle PPR solver is still an iterative nonconvex method. The
exact moment formulation gives a useful theoretical target, but it does not
ensure that a finite number of particles and restarts will find the right
basin. The hard $n=15$ instance illustrates the issue: random initialization
misses the exact optimum in all $21$ restarts, while CMG finds it directly
(Table~\ref{tab:hard_instance}). This motivates treating initialization as a
central part of the solver rather than as an implementation detail.

\begin{table}[htbp!]
\small\centering
\caption{\small Hard instance $n=15$, seed $=44$, known optimum $38$. ``CMG'' means
  CMG followed by local-flip refinement.}
\label{tab:hard_instance}
\begin{tabular}{lrrc}\toprule
Strategy & Cut & Ratio & Time \\\midrule
Random-init PPR, best of 21 & 37 & 0.974 & $21\times\!\sim\!7$s \\
Random-init PPR, exact hits & 0/21 & -- & -- \\
CMG$+$flip & \textbf{38} & \textbf{1.000} & $30\,\mu$s \\
CMG$\to$PPR warm start & \textbf{38} & \textbf{1.000} & $\sim$12s \\
\bottomrule
\end{tabular}
\end{table}

\subsection{Three complementary tools}
\label{app:three_tools}

The pipeline combines three stages. Spielman-Srivastava sparsification
\citep{spielman2011} reduces dense graphs to sparse spectral approximations
when appropriate. CMG~\citep{langston2019cmg,koutis2011combinatorial} builds
a hierarchy from the graph's own edge weights and returns a binary warm start.
PPR then uses that warm start inside a continuous particle representation that
can make multi-node changes. The composition is
\[
  \underbrace{G\xrightarrow{\text{S-S}}G'}_{\text{sparsify}}\to
  \underbrace{G'\xrightarrow{\text{CMG}}x_{\mathrm{cmg}}}_{\text{warm start}}\to
  \underbrace{x_{\mathrm{cmg}}\xrightarrow{\text{PPR}}x'}_{\text{fine-level polish}}\to
  \underbrace{x'\xrightarrow{\text{2-opt}}x^*}_{\text{refine}},
\]
matching Figure~\ref{fig:pipeline_spine}. In our experiments, the strongest performance is obtained when the stages are composed in this order; removing individual stages generally reduces solution quality on the hardest instances; however, on graphs that are already sparse, the first stage is typically unnecessary.

\subsection{The pipeline end to end}
\label{app:pipeline_end}

Figure~\ref{fig:pipeline_spine} draws the composed architecture: S-S
sparsification (a no-op on already-sparse graphs) feeds CMG coarsening to a
directly-solvable coarsest level, with the SQP sweet-spot solve substituted
in at any level with $n_\ell\leq20$; the coarsest solution is projected back
up through local-flip refinement; PPR then runs once,
warm-started from that projection, on the finest graph; and multi-pass
restarts followed by a single 2-opt sweep close the gap to the exact
optimum. This is the pipeline used for every headline result in
Section~\ref{sec:results}, traced end to end on a real graph in
Appendix~\ref{app:walkthrough}.

\begin{figure}[htbp]
\centering
\includegraphics[width=0.75\textwidth]{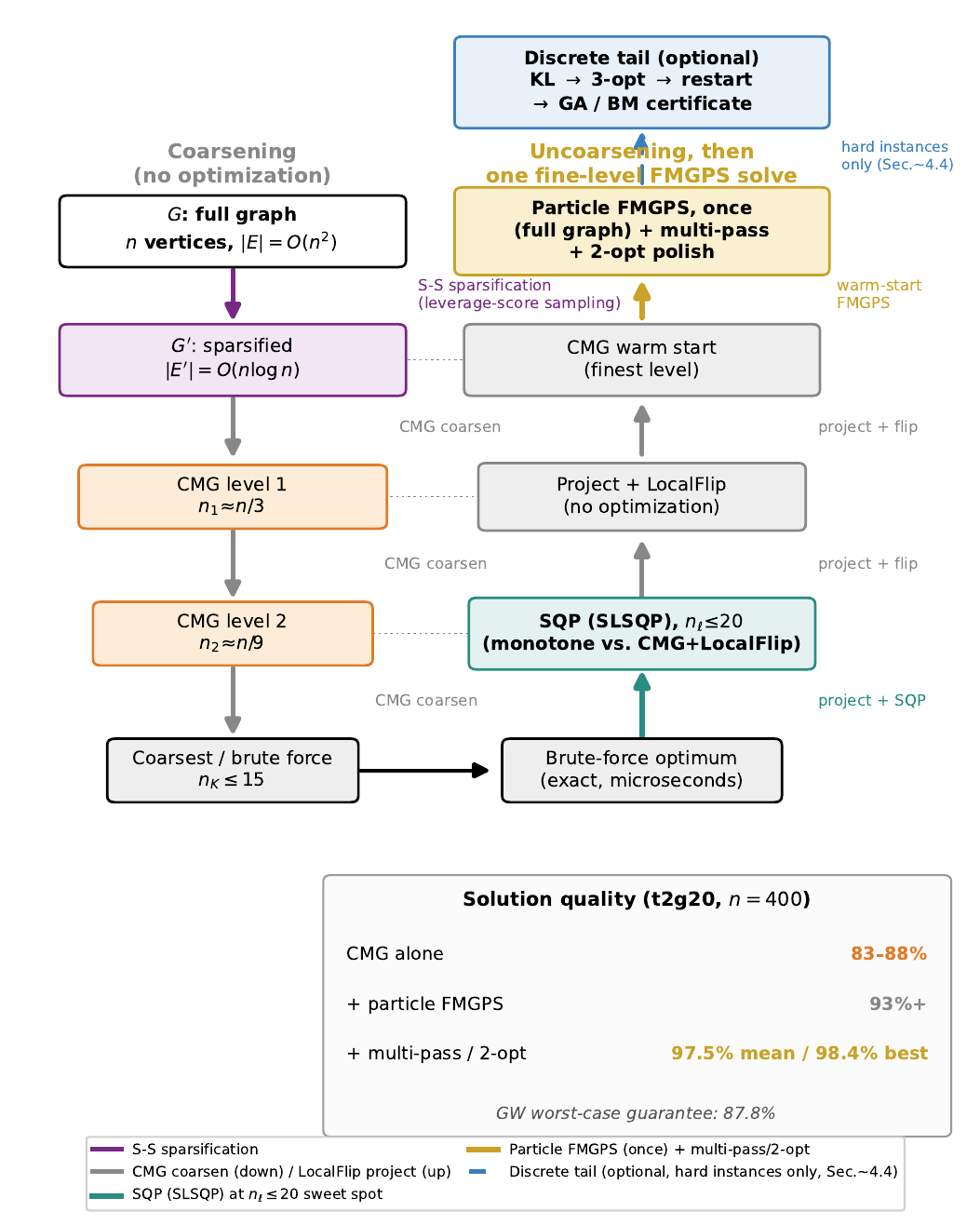}
\caption{\small The composed pipeline: S-S sparsification (no-op if already sparse)
  feeds CMG coarsening to a directly-solvable coarsest level; the coarsest
  solution is projected back up with local-flip refinement; PPR
  then runs, warm-started from that projection, on the finest graph,
  followed by multi-pass restarts and a single closing 2-opt sweep.}
\label{fig:pipeline_spine}
\end{figure}

\subsection{Further Contributions and Additional Results}
\label{app:contributions}

The main contributions are: (1) applying the particle moment framework to the
MAX-CUT binary constraint; (2) showing that CMG and PPR play complementary
roles in the compute-constrained regime, with the heavy-tail ablation
(Table~\ref{tab:tail_ablation}) showing this advantage narrows, and a
random start can match or exceed it on about half the instances tested, once
a large discrete-tail budget is affordable;
(3) demonstrating that PPR can improve beyond CMG's one-flip local optimum on
structured instances; (4) reporting the \texttt{t2g20} result of $0.975$ mean and
$0.984$ best cut/OPT; (5) identifying the dense homogeneous regime where the
method adds little; (6) evaluating the effect of a stronger discrete tail; and
(7) scaling the sparse CSR implementation to $n=20{,}000$ with no observed decay
in the gain over CMG across the tested size range.  We follow below with more details on the problem formulation, pipeline details, and additional results to strengthen and highlight our findings.

\subsubsection{The binary-constraint derivation}
\label{app:full_derivation}

With $x_i\in\{-1,+1\}$, MAX-CUT is
\[
  \max_x \frac14\sum_{(i,j)\in E}w_{ij}(1-x_ix_j),
\]
or equivalently the minimization of its negative over the algebraic variety
$\mathcal V=\{x\in[-1,1]^n:x_i^2=1\ \forall i\}$.




\subsubsection{Problem dimensions and scaling}
\label{app:scaling}

With degree-SOS $=2$, $L$ levels, and $R$ particles per factor, the particle
solver uses $O(LnR)$ parameters. For fixed $L$ and $R$, this is linear in
$n$, whereas the GW SDP uses $O(n^2)$ variables and Lasserre relaxations grow
much faster.

\begin{table}[htbp!]
\small\centering
\caption{\small Variable-count scaling for PPR-MAX-CUT and standard relaxations.}
\label{tab:scaling}
\begin{tabular}{lll}\toprule
Component & Count & Notes \\\midrule
Particle locations and weights & $O(LnR)$ & linear in $n$ for fixed $L,R$ \\\midrule
GW SDP variables & $O(n^2)$ & quadratic in $n$ \\
Lasserre level $k$ SDP & $O\!\left(\binom{n+2k}{n}^{2}\right)$ & grows rapidly with $n$ \\
\bottomrule
\end{tabular}
\end{table}

\subsubsection{Sparsification and leverage scores}
\label{app:leverage}

Effective-resistance sparsification samples edges according to leverage scores
$\ell_e=w_eR_e$, where
$R_e=L^\dagger_{ii}+L^\dagger_{jj}-2L^\dagger_{ij}$ is the effective
resistance of edge $e=(i,j)$. High-leverage edges are structurally important,
while low-leverage edges can often be removed while approximately preserving
cuts~\citep{spielman2011}. Dense uniform graphs have nearly flat leverage
scores, so sparsification gives little benefit. Weighted structured graphs can
have highly skewed leverage distributions, allowing a much smaller graph to
preserve the relevant cut structure.

\begin{figure}[htbp!]
\centering
\includegraphics[width=0.9\textwidth]{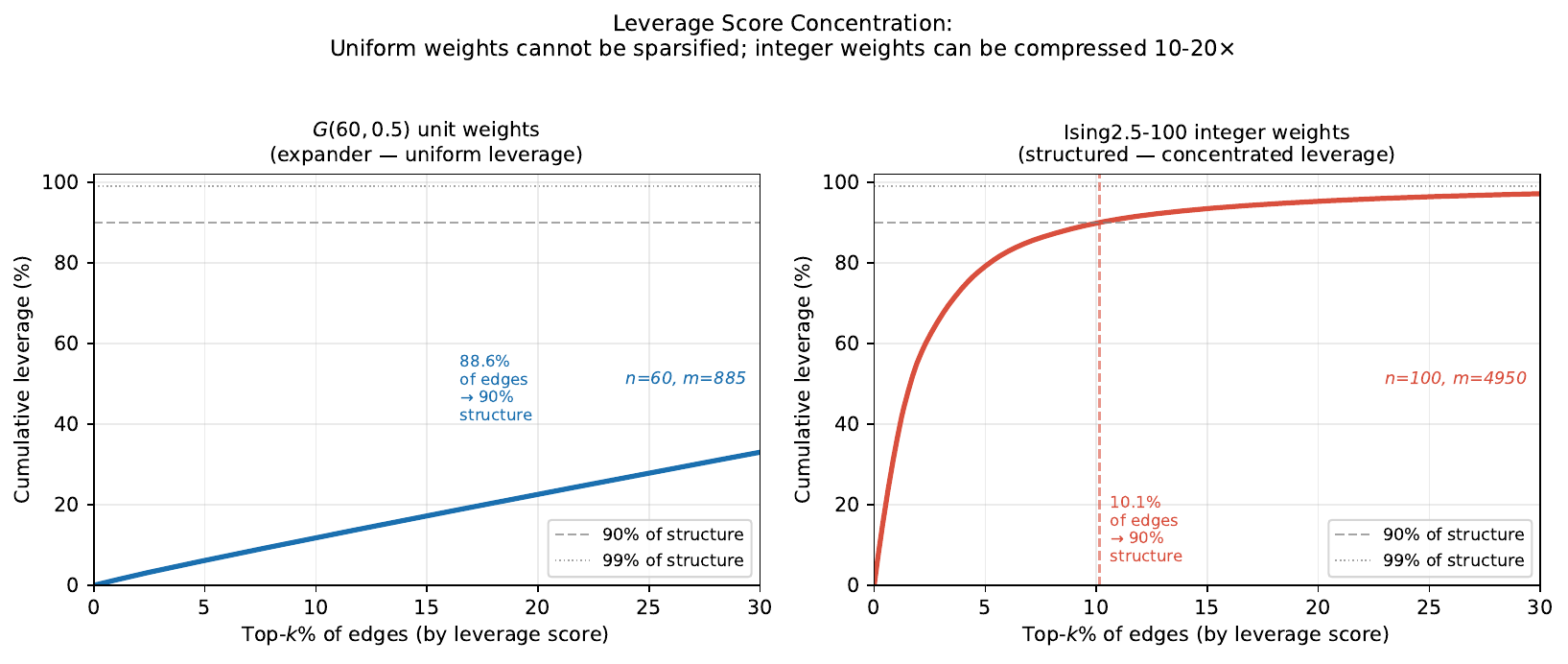}
\caption{\small Leverage-score distributions. Uniform-weight graphs have nearly flat
  leverage scores, while weighted structured graphs can have skewed scores
  that make sparsification useful.}
\label{fig:leverage}
\end{figure}

\subsubsection{CMG coarsening}
\label{app:cmg_steps}

One CMG coarsening level consists of profiling the graph, splitting long or
poorly connected forest paths, isolating weakly attached vertices, and
contracting the resulting components. Algorithm~\ref{alg:cmg_coarsen} gives
one level of this procedure.

\begin{algorithm2e}[htbp!]
\small
\SetAlgoLined
\DontPrintSemicolon
\caption{\small CMG graph coarsening, one level}
\label{alg:cmg_coarsen}
\KwIn{$G=(V,E,w)$}
\KwOut{coarsened $G'$ and cluster map $B$}
\textbf{GraphProfile:} $B[i]\gets\arg\max_{j\sim i}w_{ij}$, $M[i]\gets\max_{j\sim i}w_{ij}$, $D[i]\gets\sum_{j\sim i}w_{ij}$\;
\textbf{SplitForest:} cut edges in $B$ where path length is too large or conductance is low\;
\textbf{UpdateGroups:} if $M[i]/D[i]<1/8$, set $B[i]\gets i$\;
\textbf{ForestComponents:} label components of $B$ as clusters $\{C_1,\ldots,C_{n_c}\}$\;
\textbf{Coarsen:} $w'_{c_1c_2}\gets\sum_{(i,j)\in E,\,i\in C_{c_1},\,j\in C_{c_2}}w_{ij}$\;
\Return{coarsened $G'$ and cluster map $B$}\;
\end{algorithm2e}

Solving the coarsest graph directly and projecting the solution back down
with local-flip refinement gives the warm start $x_{\mathrm{cmg}}$. Unlike
balanced multilevel partitioners such as METIS~\citep{karypis1998} and
Graclus~\citep{dhillon2007}, CMG does not force balanced partitions. This is
useful for weighted Ising graphs whose cut structure is not balanced.

\begin{figure}[htbp!]
\centering
\includegraphics[width=0.55\textwidth]{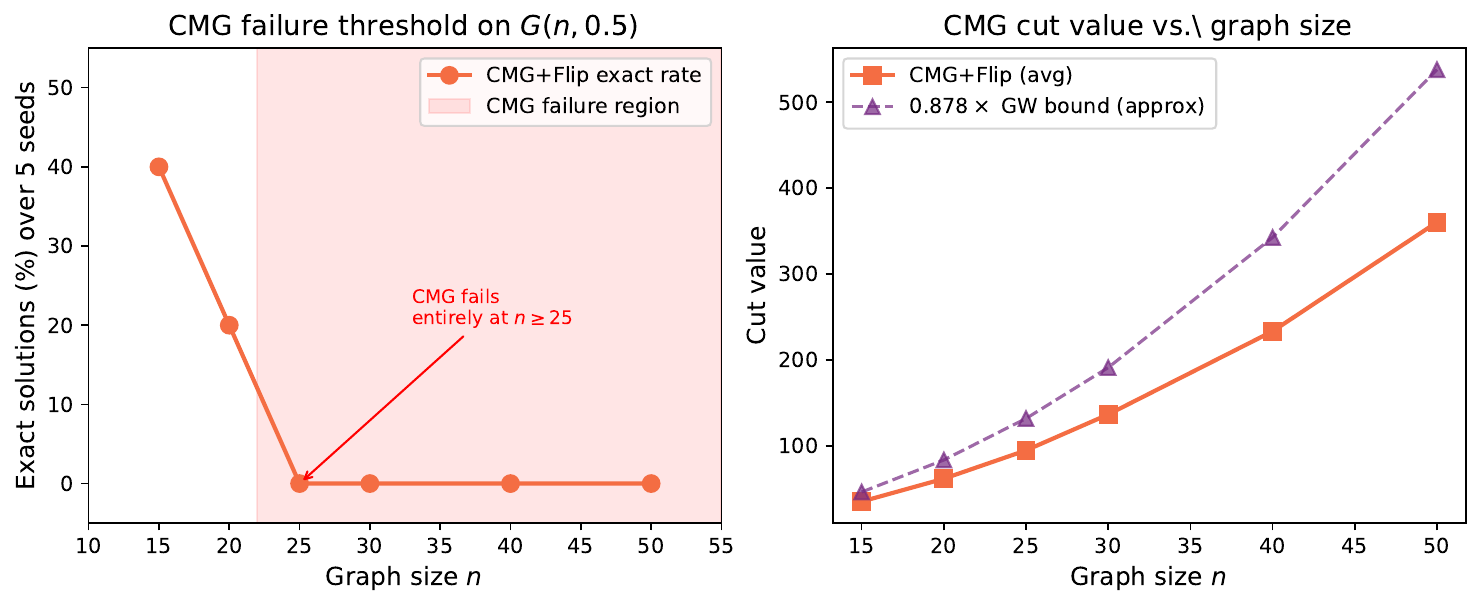}
\caption{\small CMG coarsening effectiveness versus graph density. Sparse structured
  graphs contract aggressively, while dense uniform graphs contract weakly.}
\label{fig:cmg_failure}
\end{figure}

\subsubsection{Refinement branches and multi-pass polishing}
\label{app:branches}

At very small coarsened levels ($n_\ell\leq20$), we can solve a rank-2,
degree-2-SOS objective with SLSQP \citep{kraft1988sqp} and then apply local-flip refinement. This
branch is useful only because it runs on a small coarse graph. The headline
Biq Mac and G-set results use the other branch: PPR runs directly on the
finest graph, initialized from the CMG projection. Multiple PPR passes are
then chained by feeding the best binary solution from one pass into the next,
and the final result is polished by a single 2-opt sweep on the original
graph.

\begin{table}[htbp]
\small\centering
\caption{\small Coarse sweet-spot ($n_\ell\leq20$) solver comparison, three seeds per
  size. ``Exact'' marks a cut matching the known or brute-force optimum.}
\label{tab:sqp_sweetspot}
\setlength{\tabcolsep}{4pt}
\begin{tabular}{llrrcrcrc}\toprule
Graph & Seed & OPT & CMG+Flip & Exact & ADMM & Exact & SQP & Exact\\\midrule
random-15 & 42 & 39 & 37 & N & 36 & N & \textbf{39} & Y\\
random-15 & 44 & 38 & 38 & Y & 37 & N & \textbf{38} & Y\\
random-20 & 43 & 62 & 62 & Y & 61 & N & \textbf{62} & Y\\
random-20 & 44 & 70 & 65 & N & 65 & N & \textbf{69} & N\\
planted-30-200 & 42-44 & (exact) & exact & Y & exact & Y & exact & Y\\\midrule
\multicolumn{2}{l}{Typical wall time} & & $<\!1$~ms & & 5-39~s & & 2-19~s & \\
\bottomrule
\end{tabular}
\end{table}

\subsubsection{Additional refinement and scaling results}
\label{app:additional_results}

Additional experimental results complement
the summary statistics reported in Section~\ref{sec:results}. Table~\ref{tab:t2g20} and Figures~\ref{fig:t2g20_progression}-~\ref{fig:sparse_scaling} illustrate how solution quality evolves across refinement
stages on \texttt{t2g20}, how quickly the multi-pass schedule
saturates, and how the sparse implementation scales on large G-set
instances.

\begin{table}[htbp]
\small\centering
\caption{\small \texttt{t2g20} ($n=400$) configuration ladder, mean cut/OPT over
  three seeds. The first, second, and fourth rows are the cumulative ladder
  plotted in Figure~\ref{fig:t2g20_progression}; the third is the competing
  \emph{2-opt after every pass} configuration, which branches off pass~1
  rather than extending the ladder and finishes $1.8$ points behind.}
\label{tab:t2g20}
\begin{tabular}{lcc}\toprule
Configuration & Mean & Best seed \\\midrule
CMG baseline & 0.832 & -- \\
5-pass PPR & 0.952 & 0.968 \\
2-opt after every pass & 0.957 & 0.970 \\
\textbf{5-pass $+$ single closing 2-opt} & \textbf{0.975} & \textbf{0.984} \\\midrule
GW guarantee & 0.878 & -- \\
\bottomrule
\end{tabular}
\end{table}

\begin{figure}[htbp]
\centering
\includegraphics{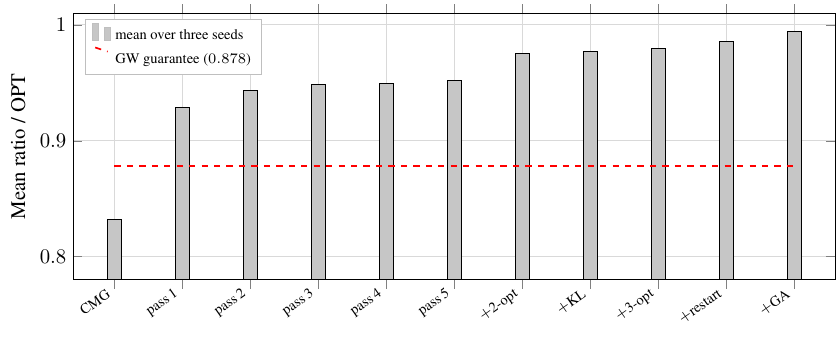}
\caption{
\small Progressive improvement on \texttt{t2g20} ($n=400$) across the
\emph{cumulative} stage ladder, mean over three seeds. The CMG warm
start opens at $0.832$; five PPR passes carry it to $0.952$; a single
closing 2-opt sweep gives the $0.975$ headline of
Table~\ref{tab:t2g20}; and the discrete tail (KL, 3-opt, $200$-restart,
memetic GA) closes to $0.994$. Every stage shown here is applied on top
of the previous one. The competing \emph{2-opt after every pass}
configuration is deliberately absent: it branches off pass~1 rather
than extending the ladder, and is reported in Table~\ref{tab:t2g20} above
and per-instance in Table~\ref{tab:2opt_order}.
}
\label{fig:t2g20_progression}
\end{figure}

\begin{figure}[htbp]
\centering
\includegraphics{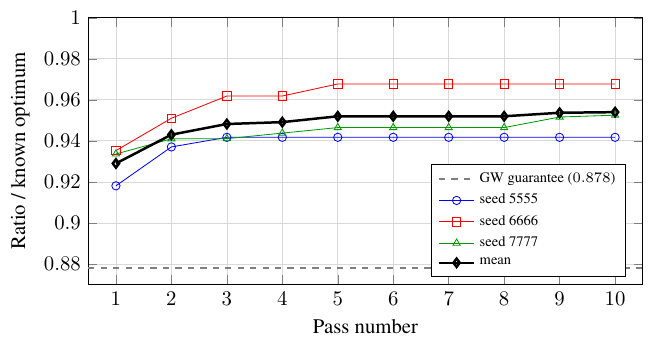}
\caption{
\small Ratio to known optimum versus pass number on \texttt{t2g20} ($n=400$),
$L{=}128$, $R{=}32$, $k{=}40$. Pass~2 yields the largest single gain
($+0.014$); the curve is essentially flat from pass~5 onward. Read
against Table~\ref{tab:phase8}, where \emph{quadrupling} the restart
count buys only $+0.004$ in total: it is the multi-pass schedule, not
raw restart breadth, that escapes the warm-start basin.
}
\label{fig:pass_conv}
\end{figure}

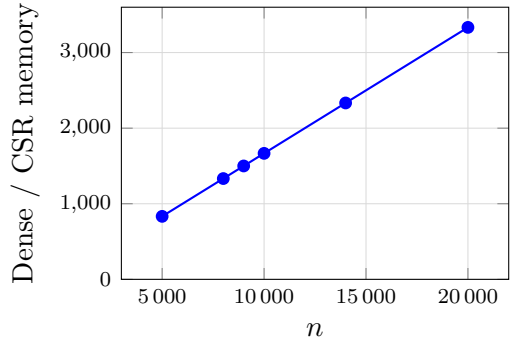
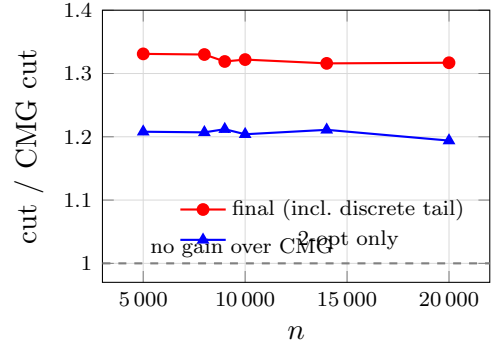
\begin{figure}[htbp]
\centering

\subfigure[Dense/CSR memory ratio; grows linearly in $n$.]{%
\begin{tikzpicture}
\begin{axis}[
  width=0.44\textwidth, height=0.34\textwidth,
  xlabel={$n$}, ylabel={Dense / CSR memory},
  xmin=3000, xmax=22000, ymin=0, ymax=3600,
  scaled x ticks=false, xtick={5000,10000,15000,20000},
  xticklabel style={/pgf/number format/fixed,
                    /pgf/number format/1000 sep={\,}, font=\scriptsize},
  yticklabel style={font=\scriptsize},
  grid=major, grid style={line width=0.3pt, draw=gray!30},
]
\addplot[solid, blue, mark=*, thick] coordinates {
  (5000,833.3) (8000,1333.3) (9000,1500.0)
  (10000,1666.7) (14000,2333.3) (20000,3333.3)};
\end{axis}
\end{tikzpicture}
\label{fig:sparse_memory}
}
\hfill
\subfigure[Final/CMG and 2-opt-only/CMG; both flat in $n$.]{%
\begin{tikzpicture}
\begin{axis}[
  width=0.44\textwidth, height=0.34\textwidth,
  xlabel={$n$}, ylabel={cut / CMG cut},
  xmin=3000, xmax=22000, ymin=0.97, ymax=1.4,
  scaled x ticks=false, xtick={5000,10000,15000,20000},
  xticklabel style={/pgf/number format/fixed,
                    /pgf/number format/1000 sep={\,}, font=\scriptsize},
  ytick={1.0,1.1,1.2,1.3,1.4}, yticklabel style={font=\scriptsize},
  grid=major, grid style={line width=0.3pt, draw=gray!30},
  legend style={font=\scriptsize, at={(0.97,0.08)}, anchor=south east,
                draw=none, fill=none},
]
\addplot[solid, red, mark=*, thick] coordinates {
  (5000,1.331) (8000,1.330) (9000,1.319)
  (10000,1.322) (14000,1.316) (20000,1.317)};
\addlegendentry{final (incl.\ discrete tail)}
\addplot[solid, blue, mark=triangle*, thick] coordinates {
  (5000,1.208) (8000,1.207) (9000,1.212)
  (10000,1.204) (14000,1.211) (20000,1.194)};
\addlegendentry{2-opt only}
\addplot[dashed, thick, gray] coordinates {(3000,1.0)(22000,1.0)};
\node[anchor=south west, font=\scriptsize, inner sep=1pt]
  at (axis cs:5200,1.005) {no gain over CMG};
\end{axis}
\end{tikzpicture}
\label{fig:sparse_gain}
}

\caption{
\small Sparse scaling on six G-set instances, $n=5000$-$20{,}000$ (a superset
of the four shown in Table~\ref{tab:sparse}, which omits \texttt{G65}
and \texttt{G66}). Memory savings grow linearly with $n$, exceeding
$3300\times$ at $n=20{,}000$. The right-hand panel shows the gain over the
CMG warm start at both pipeline stages: the continuous-relaxation path alone
(CMG$\to$PPR$\to$2-opt, blue) holds a tight $1.194$-$1.212\times$ band, and
appending the discrete tail (KL/3-opt/GA/ILS, red) carries this to
$1.316$-$1.331\times$; neither band decays at the largest instance. Both
panels are plotted against a common $n$ axis; the
right-hand panel's vertical axis is anchored at $1.0$, the no-gain
line, so that the flatness of the bands is not exaggerated by a
magnified scale.
}
\label{fig:sparse_scaling}
\end{figure}

\clearpage
\subsubsection{Pass and restart saturation}
\label{app:pass_saturation}

Tables~\ref{tab:pass_sweep} and~\ref{tab:phase8}
below provide the evidence that the quality of the warm-start basin, not the breadth of
random restarts, is what limits the pipeline on \texttt{t2g20}. Table~\ref{tab:pass_sweep} sweeps the pass count at
fixed restart budget: pass~2 delivers the largest single gain ($+0.014$), the
curve then flattens, and five passes capture essentially all achievable gain
($0.952$), with only a small late bump at pass~9. Table~\ref{tab:phase8} holds
the pass count at one and instead \emph{quadruples} the restart count: going
$k{=}40\to80\to160$ buys $+0.003$ then $+0.001$, so a single additional pass is
worth roughly five times a doubling of $k$.

\begin{table}[htbp]
\small\centering
\caption{\small Pass-count saturation on \texttt{t2g20} ($n=400$).
  $L{=}128$, $R{=}32$, $k{=}40$, EXTENDED penalty, $s{=}0.5$.
  $\Delta$ is the incremental mean gain over the previous pass.}
\label{tab:pass_sweep}
\begin{tabular}{rccccr}\toprule
Pass & seed 5555 & seed 6666 & seed 7777 & Mean & $\Delta$ \\\midrule
1  & 0.9180 & 0.9352 & 0.9336 & 0.9289 & --     \\
2  & 0.9370 & 0.9510 & 0.9410 & 0.9430 & +0.0140 \\
3  & 0.9417 & 0.9618 & 0.9410 & 0.9482 & +0.0052 \\
4  & 0.9417 & 0.9618 & 0.9438 & 0.9491 & +0.0009 \\
5  & 0.9417 & 0.9677 & 0.9464 & 0.9520 & +0.0029 \\
6  & 0.9417 & 0.9677 & 0.9464 & 0.9520 & +0.0000 \\
7  & 0.9417 & 0.9677 & 0.9464 & 0.9520 & +0.0000 \\
8  & 0.9417 & 0.9677 & 0.9464 & 0.9520 & +0.0000 \\
9  & 0.9417 & 0.9677 & 0.9516 & 0.9537 & +0.0017 \\
10 & 0.9417 & 0.9677 & 0.9525 & 0.9540 & +0.0003 \\\bottomrule
\end{tabular}
\end{table}

\begin{table}[htbp]
\small\centering
\caption{\small Single-pass high-$k$ restart sweep on \texttt{t2g20}.
  $L{=}128$, $R{=}32$, EXTENDED penalty, $s{=}0.5$. Quadrupling $k$ gains
  less than one extra pass does in Table~\ref{tab:pass_sweep}.}
\label{tab:phase8}
\begin{tabular}{rccccr}\toprule
$k$ & seed 5555 & seed 6666 & seed 7777 & Mean & $\Delta$ \\\midrule
 40 & 0.9180 & 0.9352 & 0.9336 & 0.9289 & --    \\
 80 & 0.9175 & 0.9475 & 0.9309 & 0.9320 & +0.003 \\
160 & 0.9249 & 0.9424 & 0.9323 & 0.9332 & +0.001 \\\bottomrule
\end{tabular}
\end{table}

\subsubsection{The full \texttt{t2g20} stage ladder and 2-opt ordering}
\label{app:t2g20_ladder}

Table~\ref{tab:2opt_order} isolates the 2-opt sweep itself. Applied once after
five passes (variant~A) it raises the mean from $0.952$ to $\mathbf{0.975}$ on
only $12$-$19$ accepted flips, and the best single seed reaches
$\mathbf{0.984}$ of the exact optimum. Interleaving it after every pass
(variant~B) instead finds \emph{zero} improving flips: the solution is already
2-opt-optimal from pass~1 onward, locked into a basin later passes cannot
escape, and it finishes at $0.957$. Variant~C shows how much a single 2-opt
sweep is worth: $1$-pass $+$ 2-opt ($0.962$) outperforms $5$-pass alone
($0.952$), so one 2-bit sweep recovers more than four extra particle passes.
Table~\ref{tab:configs} orders every configuration by mean ratio, and
Table~\ref{tab:t2g20_tail} extends the winning one with the discrete tail.

\begin{table}[htbp]
\small\centering
\caption{\small 2-opt polish on \texttt{t2g20} ($n=400$). ``Before''/``After'' are
  ratio/OPT around the 2-opt sweep; ``Flips'' counts accepted 2-opt moves.}
\label{tab:2opt_order}
\begin{tabular}{llccrr}\toprule
Experiment & Seed & Before & After & $\Delta$ & Flips \\\midrule
\multirow{4}{*}{A: 5-pass $+$ 2-opt}
  & 5555 & 0.9417 & 0.9664 & +0.025 & 12 \\
  & 6666 & 0.9677 & \textbf{0.9843} & +0.017 & 19 \\
  & 7777 & 0.9464 & 0.9746 & +0.028 & 18 \\
  & mean & 0.9520 & \textbf{0.9751} & +0.023 & 16 \\\midrule
\multirow{4}{*}{B: 2-opt each pass}
  & 5555 & 0.9548 & 0.9548 & 0.000 & 0 \\
  & 6666 & 0.9700 & 0.9700 & 0.000 & 0 \\
  & 7777 & 0.9461 & 0.9461 & 0.000 & 0 \\
  & mean & 0.9570 & 0.9570 & 0.000 & 0 \\\midrule
\multirow{4}{*}{C: 1-pass $+$ 2-opt}
  & 5555 & 0.9217 & 0.9637 & +0.042 & 23 \\
  & 6666 & 0.9498 & 0.9653 & +0.015 & 13 \\
  & 7777 & 0.9276 & 0.9574 & +0.030 & 24 \\
  & mean & 0.9330 & \textbf{0.9621} & +0.029 & 20 \\\midrule
\multicolumn{2}{l}{GW guarantee} & \multicolumn{4}{c}{0.878} \\
\bottomrule
\end{tabular}
\end{table}

\begin{table}[htbp]
\small\centering
\caption{\small All \texttt{t2g20} ($n=400$) configurations, ordered by mean ratio/OPT.
  ``1-pass only'' is the single-pass configuration with 2-opt disabled;
  Table~\ref{tab:pass_sweep}'s pass-1 row ($0.9289$) is the same nominal
  configuration measured in a separate run and differs by $0.0041$; this is a run-to-run variation, not a discrepancy in the underlying method.}
\label{tab:configs}
\begin{tabular}{lcc}\toprule
Configuration & Mean & Best seed \\\midrule
CMG baseline               & 0.832 & -- \\
1-pass only                & 0.933 & 0.950 \\
5-pass only                & 0.952 & 0.968 \\
B: 2-opt after each pass   & 0.957 & 0.970 \\
\textbf{C: 1-pass $+$ 2-opt} & \textbf{0.962} & 0.965 \\
\textbf{A: 5-pass $+$ 2-opt} & \textbf{0.975} & \textbf{0.984} \\\midrule
GW guarantee               & 0.878 & -- \\
\bottomrule
\end{tabular}
\end{table}

\begin{table}[htbp]
\small\centering
\caption{\small Discrete tail appended to the Experiment~A (5-pass PPR $+$ 2-opt)
  partition on \texttt{t2g20}, all three seeds, FULL config
  ($L{=}128,R{=}32,k{=}40$, RHO\_EXTENDED, $s{=}0.5$). The $+$2-opt row
  reproduces Table~\ref{tab:2opt_order}'s Experiment~A exactly (bit-identical
  RNG stream and draw order), confirming the tail is layered on the genuine
  headline partition, not an approximation of it. \texttt{ga\_only} is the
  uncontaminated baseline (independent RNG, no CMG or PPR) at the same GA
  budget; note that it \emph{exceeds} the full pipeline here, the same
  finding the heavy-set ablation of Table~\ref{tab:tail_ablation} reports.}
\label{tab:t2g20_tail}
\begin{tabular}{lccccc}\toprule
Stage & 5555 & 6666 & 7777 & Mean & vs.\ headline \\\midrule
CMG              & 0.7915 & 0.8630 & 0.8406 & 0.8317 & -- \\
5-pass PPR       & 0.9417 & 0.9677 & 0.9464 & 0.9520 & $=$Table~\ref{tab:configs} \\
$+$2-opt (headline) & 0.9664 & 0.9843 & 0.9746 & 0.9751 & $=$Table~\ref{tab:2opt_order} \\
$+$KL            & 0.9710 & 0.9843 & 0.9746 & 0.9766 & +0.0015 \\
$+$3-opt         & 0.9723 & 0.9888 & 0.9778 & 0.9796 & +0.0046 \\
$+$restart (200) & 0.9846 & 0.9888 & 0.9836 & 0.9857 & +0.0106 \\
\textbf{$+$GA}   & \textbf{0.9876} & \textbf{0.9959} & \textbf{0.9988} & \textbf{0.9941} & \textbf{+0.0190} \\\midrule
\texttt{ga\_only} (uncontaminated) & 0.9971 & 0.9973 & 0.9986 & 0.9977 & -- \\
\bottomrule
\end{tabular}
\end{table}

\subsubsection{Tail ablation on the heavy set, and where the \texttt{G81} budget goes}
\label{app:tail_anatomy}

This block supports the open question: ``Once a discrete tail is
budgeted, does the warm start still decide the answer?''
Table~\ref{tab:tail_ablation} ablates the tail stage by stage across eight
instances, and Figure~\ref{fig:tail_progression} plots it against the
\texttt{ga\_only} control: a pure-random-seeded GA at the same budget, with
no CMG or PPR. On three of eight instances the control \emph{beats} the full
pipeline and on a fourth it ties, which is the isolation finding the body
cites. Table~\ref{tab:tail_ablation} reports absolute cuts, so nothing in it
depends on an external reference value. Figure~\ref{fig:tail_progression}
normalizes the same measurements by a best-known cut drawn from the
literature.  Table~\ref{tab:tail_bm} adds
an \emph{empirical} SDP-relaxation estimate (not a certified upper bound) for the same eight instances.


\begin{table}[htbp]
\small\centering
\caption{\small Discrete-tail ablation across the heavy set (single core, one shared
  seed so every row is directly comparable). The base of each row is the CMG
  warm start $\to$ light multi-pass PPR ($L{=}32,R{=}8,k{=}20$, five
  passes) $\to$ 2-opt; subsequent columns append KL, then $200$-restart local
  search, then the memetic GA, reporting the cut after each. 3-opt runs between
  KL and the restart stage but is omitted as a column: it finds \emph{zero}
  improving moves on seven of the eight instances, and $+2$ units on
  \texttt{G22}. \texttt{ga\_only} is the control: the same memetic GA at the
  same budget from a pure random seed, with no CMG and no PPR. All cuts are
  absolute, so the pipeline-versus-control comparison needs no external
  reference value; \textbf{bold} marks the larger of $+$GA and
  \texttt{ga\_only} on each row. $t_{\mathrm{F}}$ and $t_{\mathrm{tail}}$ are
  the PPR and total-tail wall-clock in seconds. The top block is the
  sparse-signed toroidal target regime; the bottom block is dense
  uniform-weight contrast.}
\label{tab:tail_ablation}
\setlength{\tabcolsep}{4.5pt}
{%
\begin{tabular}{lrr rrrrr r r}\toprule
Instance & $n$ & $|E|$ & CMG & PPR$+$2-opt & $+$KL & $+$rst & $+$GA & \texttt{ga\_only} & $t_{\mathrm{F}}/t_{\mathrm{tail}}$\\\midrule
\multicolumn{10}{l}{\emph{Sparse-signed toroidal (target regime)}}\\
\texttt{G11} & 800  & 1600  & 418  & 532  & 548  & 560  & 560            & \textbf{562}   & 133/25\\
\texttt{G12} & 800  & 1600  & 414  & 528  & 546  & 550  & 554            & \textbf{556}   & 163/25\\
\texttt{G13} & 800  & 1600  & 416  & 546  & 562  & 578  & \textbf{580}   & \textbf{580}   & 165/26\\
\texttt{G32} & 2000 & 4000  & 1070 & 1306 & 1356 & 1382 & \textbf{1402}  & 1396           & 550/128\\
\texttt{G33} & 2000 & 4000  & 1012 & 1298 & 1328 & 1358 & \textbf{1368}  & 1362           & 580/130\\
\texttt{G34} & 2000 & 4000  & 1016 & 1300 & 1344 & 1366 & 1372           & \textbf{1376}  & 631/132\\\midrule
\multicolumn{10}{l}{\emph{Dense uniform-weight (contrast)}}\\
\texttt{G1}  & 800  & 19176 & 11300 & 11541 & 11541 & 11571 & \textbf{11621} & 11617 & 235/513\\
\texttt{G22} & 2000 & 19990 & 12806 & 13113 & 13126 & 13228 & \textbf{13355} & 13332 & 1068/305\\
\bottomrule
\end{tabular}%
}
\end{table}

\begin{figure}[htbp]
\centering
\includegraphics{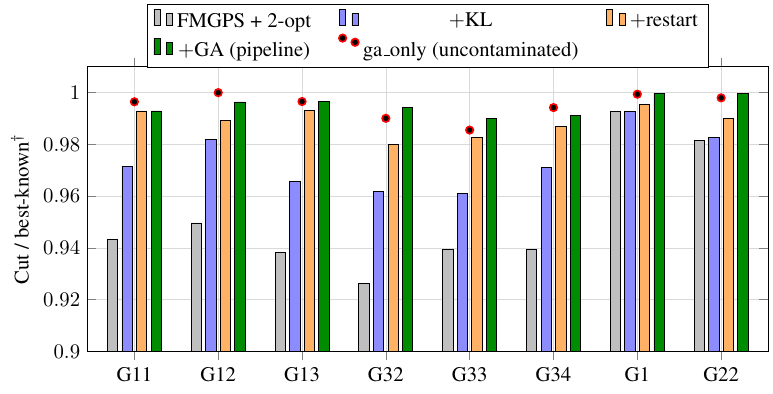}
\caption{\small Cut/best-known$^\dagger$ after each tail stage
  (Table~\ref{tab:tail_ablation}), with the uncontaminated \texttt{ga\_only}
  baseline overlaid as red points. Where the red point sits \emph{above} the
  green $+$GA bar (\texttt{G11}, \texttt{G12}, \texttt{G34}; equal on
  \texttt{G13}), a pure-random-seeded GA with no CMG or PPR warm start reaches a
  better final cut than the full pipeline; the visual form of the isolation
  finding: once the discrete tail is budgeted, the particle solver does not
  uniformly improve the reported answer. $^\dagger$\footnotesize{The denominator here is an
  external best-known reference, not a proven optimum; its provenance is
  flagged in the text preceding Table~\ref{tab:tail_ablation}.
  It cancels from the red-versus-green comparison this figure is drawn to make,
  since both are our own measurements on the same instance.}}
\label{fig:tail_progression}
\end{figure}

Table~\ref{tab:g81_stages} asks where the wall clock actually goes on the
largest instance. The marginal efficiency column spans \emph{five orders of
magnitude}: the first five stages consume $75.7$\,s, or $1.53\%$ of the
original budget, and deliver most of the total gain, while the $200$-restart
stage spends $2782$\,s to buy $30$ cut units: $0.011$ cut/s, the least
productive stage in the study by three orders of magnitude. Removing it and
retuning the remaining stages is what reduces the budget from $4951$\,s to
$261$\,s while \emph{improving} the cut.

\begin{table}[htbp]
\small\centering
\caption{\small Per-stage anatomy of \texttt{G81} ($n=20{,}000$, $|E|=40{,}000$,
  optimum $14060$~\citep{zick2026cosm}) before and after retuning.
  ``$\Delta$cut'' is the gain over the preceding stage; ``$t$'' is that
  stage's own wall time (obtained by differencing the harness's cumulative
  timestamps); ``cut/s'' is $\Delta\mathrm{cut}/t$, the stage's marginal
  efficiency. The original column is an earlier, unretuned configuration:
  dense all-pairs 2-opt, linear-scan
  Kernighan-Lin, $200$ random restarts, memetic GA at $150$
  generations. The retuned column, CSR-adjacent 2-opt, heap-based KL, no
  restart stage, GA at $800$ generations, and a perturbation-based iterated
  local search, is the configuration behind Table~\ref{tab:sparse};
  ``--'' marks a stage that configuration does not run, and the
  retuned 2-opt's efficiency is omitted because its wall time is at timer
  resolution.}
\label{tab:g81_stages}
\setlength{\tabcolsep}{4.5pt}
\begin{tabular}{lrrrrrrr}\toprule
& \multicolumn{3}{c}{Original ($4951$\,s)} & \multicolumn{3}{c}{Retuned ($261.2$\,s)} & \\\cmidrule(lr){2-4}\cmidrule(lr){5-7}
Stage & cut & $t$ (s) & cut/s & cut & $t$ (s) & cut/s & \% opt\\\midrule
CMG warm start          & 10574 & 0.004   & --            & 10574 & 0.007  & --   & 75.21\%\\
Light PPR (3 passes)    & 12466 & 64.37   & 29.4           & 12466 & 64.07  & 29.5  & 88.66\%\\
2-opt sweep             & 12628 & 4.47    & 36.3           & 12626 & 0.000  & --   & 89.80\%\\
Kernighan-Lin          & 13674 & 6.19    & \textbf{169.0} & 13674 & 0.08   & \textbf{13789.5} & 97.25\%\\
3-opt $+$ KL/3-opt loop & 13676 & 0.70    & 2.9            & 13676 & 0.04   & 48.8  & 97.27\%\\
Multi-restart ($200$)   & 13706 & 2781.86 & \textbf{0.011} & --   & --    & --   & --\\
Memetic GA              & 13876 & 1880.56 & 0.090          & 13912 & 57.50  & 4.10  & 98.95\%\\
Iterated local search   & --   & --     & --            & 13922 & 136.11 & 0.073 & \textbf{99.02\%}\\
Burer-Monteiro round   & 13876 & 213.10  & 0.000          & 13922 & 3.39   & 0.000 & 99.02\%\\
\bottomrule
\end{tabular}
\end{table}

Finally, every stage above chases a bigger \emph{cut}, and none of them says
how close that cut is to optimal. Table~\ref{tab:tail_bm} closes the ablation
with a stage that answers the other question: solve the rank-$p$ SDP relaxation
$\max\tfrac12\sum w_{ij}(1-\langle v_i,v_j\rangle)$, $\lVert v_i\rVert=1$, by
the Mixing Method~\citep{wang2017mixing} (coordinate update
$v_i\leftarrow-g_i/\lVert g_i\rVert$), round via $200$ GW random hyperplanes,
and locally optimize each rounding. At $p=\lceil\sqrt{2n}\rceil$
\citep{boumal2016burermonteiro} the rank-restricted relaxation is
\emph{generically} free of spurious local minima, but the Mixing Method's
convergence check does not itself verify that condition, so the converged
value is an \emph{empirical estimate}, not a certified upper bound on the
true MAX-CUT value (Table~\ref{tab:tail_bm} and
Appendix~\ref{app:walkthrough} say more). The resulting gaps, $10$-$12\%$ on the
sparse-signed instances and $4$-$6\%$ on the dense contrast pair, are
orthogonal to the cut-improvement stages: Burer-Monteiro never beats the GA on
cut value, and is not intended to.

\begin{table}[htbp]
\small\centering
\caption{\small Burer-Monteiro relax-and-reround estimate (rank
  $p=\lceil\sqrt{2n}\rceil$, $200$ GW hyperplanes, single core). ``bound'' is
  the converged value of the Mixing Method's coordinate updates on the
  (nonconvex, sphere-constrained) SDP relaxation; since the Mixing Method
  is not guaranteed to reach the relaxation's global optimum, this is an
  \emph{empirical estimate}, not a certified upper bound on the true
  MAX-CUT value (Appendix~\ref{app:walkthrough}'s Karate Club walkthrough
  reports a small negative gap, which a genuine upper bound cannot produce).
  gap$_{\mathrm{BM}}$ is the empirical gap between that estimate and
  Burer-Monteiro's own rounded cut; gap$_{\mathrm{GA}}$ is the empirical gap
  between the estimate and the pipeline's best reported cut (the $+$GA column
  of
  Table~\ref{tab:tail_ablation}), which is always at least as tight since the
  GA cut is $\geq$ BM's cut on every instance. Both gap columns are measured
  against an estimate this pipeline computes itself, so like
  Table~\ref{tab:tail_ablation} they involve no external reference value.}
\label{tab:tail_bm}
\begin{tabular}{lrrrrrr}\toprule
Instance & bound & BM cut & GA cut & gap$_{\mathrm{BM}}$ (\%) & gap$_{\mathrm{GA}}$ (\%) & $t_{\mathrm{BM}}$ (s)\\\midrule
\texttt{G11} & 627.7   & 558   & 560   & 11.10 & 10.78 & 1.9\\
\texttt{G12} & 622.3   & 554   & 554   & 10.97 & 10.97 & 1.9\\
\texttt{G13} & 645.1   & 578   & 580   & 10.40 & 10.09 & 1.9\\
\texttt{G32} & 1563.1  & 1386  & 1402  & 11.33 & 10.31 & 12.0\\
\texttt{G33} & 1539.5  & 1360  & 1368  & 11.66 & 11.14 & 12.1\\
\texttt{G34} & 1541.7  & 1366  & 1372  & 11.39 & 11.00 & 11.9\\
\texttt{G1}  & 12077.6 & 11591 & 11621 & 4.03  & 3.78  & 108.8\\
\texttt{G22} & 14125.8 & 13283 & 13355 & 5.97  & 5.46  & 83.1\\
\bottomrule
\end{tabular}
\end{table}

\subsubsection{Biq Mac families: cut quality against wall time}
\label{app:biqmac_fig}

Figure~\ref{fig:biqmac_neural} is the visual form of body
Table~\ref{tab:biqmac}.

\begin{figure}[htbp]
\centering
\includegraphics[width=\textwidth]{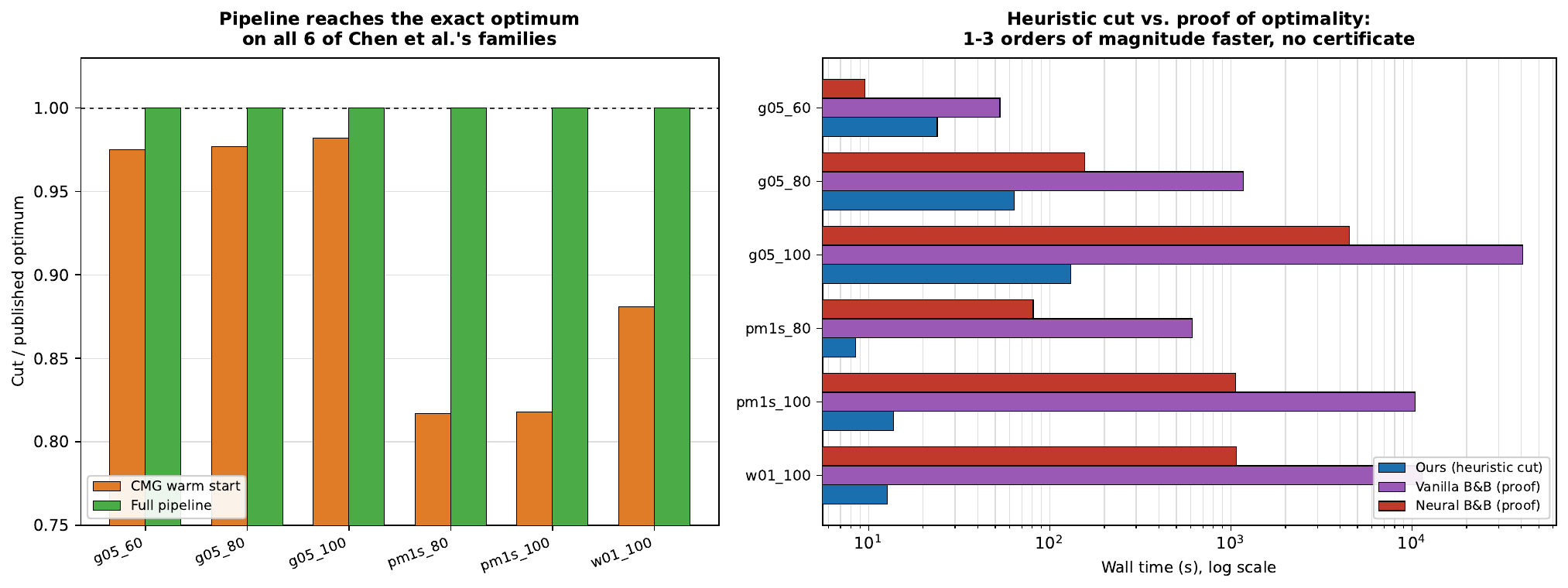}
\caption{\small The six Biq Mac families of Chen et al.~\citep{chen2026neuralsdp},
  ten instances each. \textbf{Left:} cut divided by published optimum, CMG
  warm start versus the full pipeline; every family reaches $1.000$.
  \textbf{Right:} wall time on a log axis, ours against their vanilla and
  neural branch-and-bound. The two axes measure different deliverables: our
  bar is a heuristic cut that matches a published optimum, theirs is a
  \emph{proof} of optimality. The time comparison is therefore not a
  like-for-like ranking of solvers, and is shown because it is the comparison
  the learned-solver literature sets up.}
\label{fig:biqmac_neural}
\end{figure}

\subsubsection{Pipeline Walkthrough Example}
\label{app:walkthrough}

Figure~\ref{fig:walkthrough} shows the pipeline on Zachary's Karate
Club~\citep{zachary1977}. The example is illustrative rather than a benchmark:
it uses degree-weighted edges to expose structure and does not report a ratio
to a published optimum. It shows the mechanical role of each stage: sparsify,
coarsen, project a warm start, and run PPR on the finest graph.

\begin{figure}[htbp]
\centering
\subfigure[Input $G$]{\includegraphics[width=0.46\textwidth]{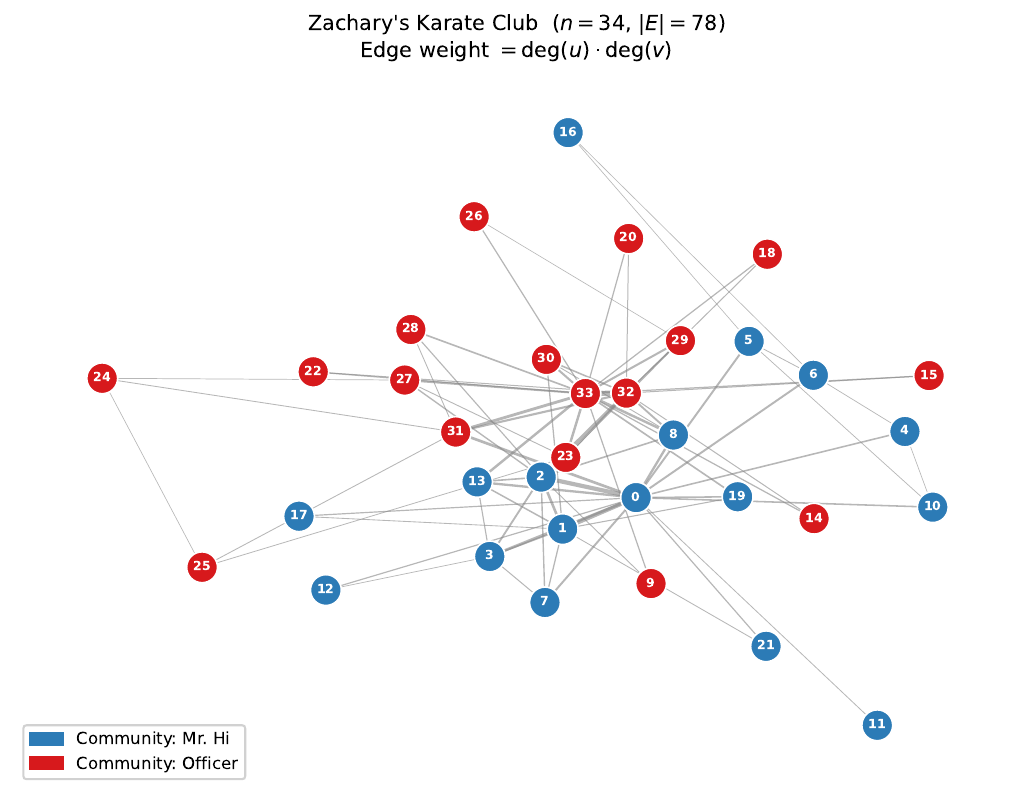}}\hfill
\subfigure[After S-S]{\includegraphics[width=0.46\textwidth]{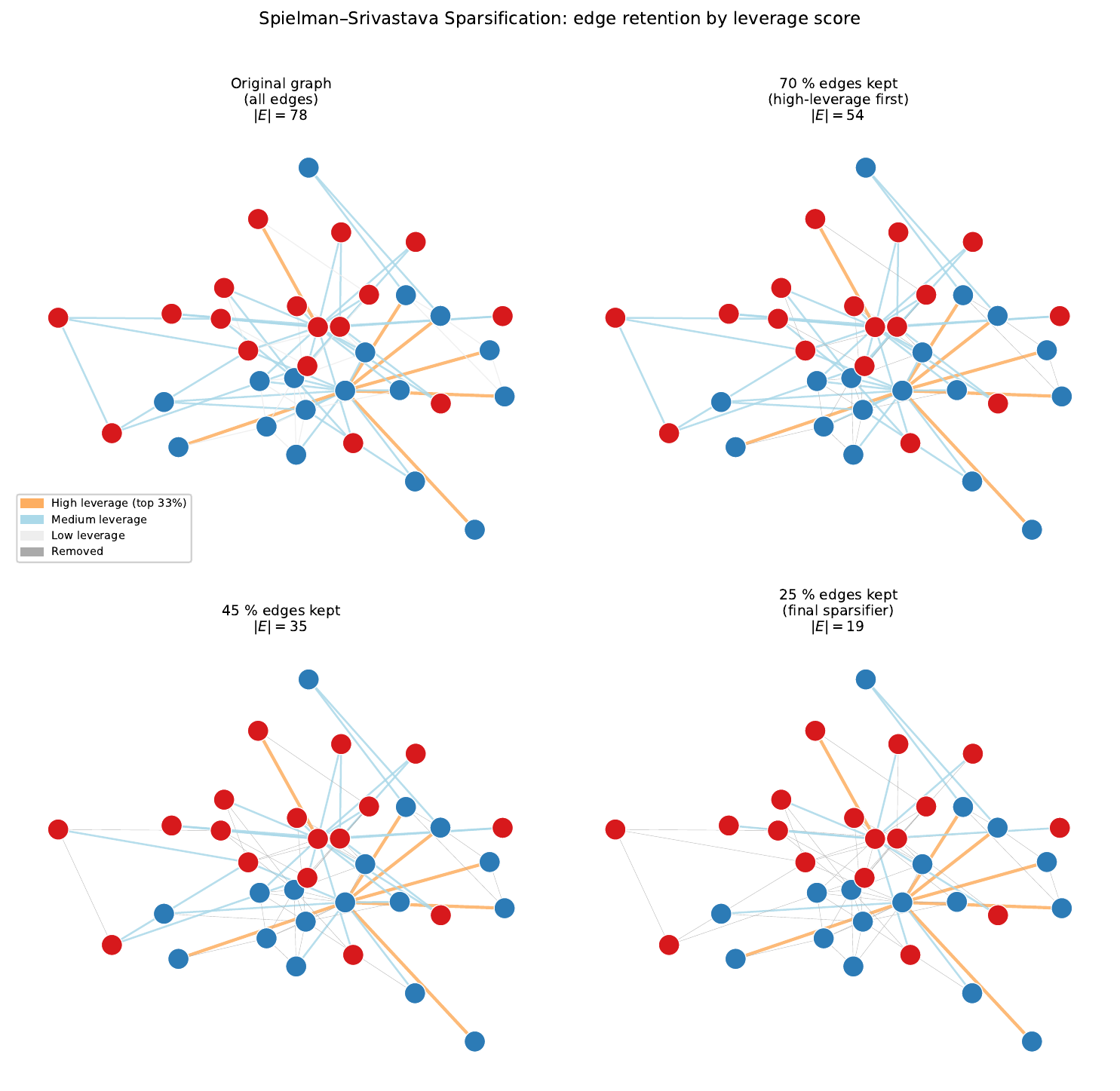}}

\subfigure[CMG groups]{\includegraphics[width=0.46\textwidth]{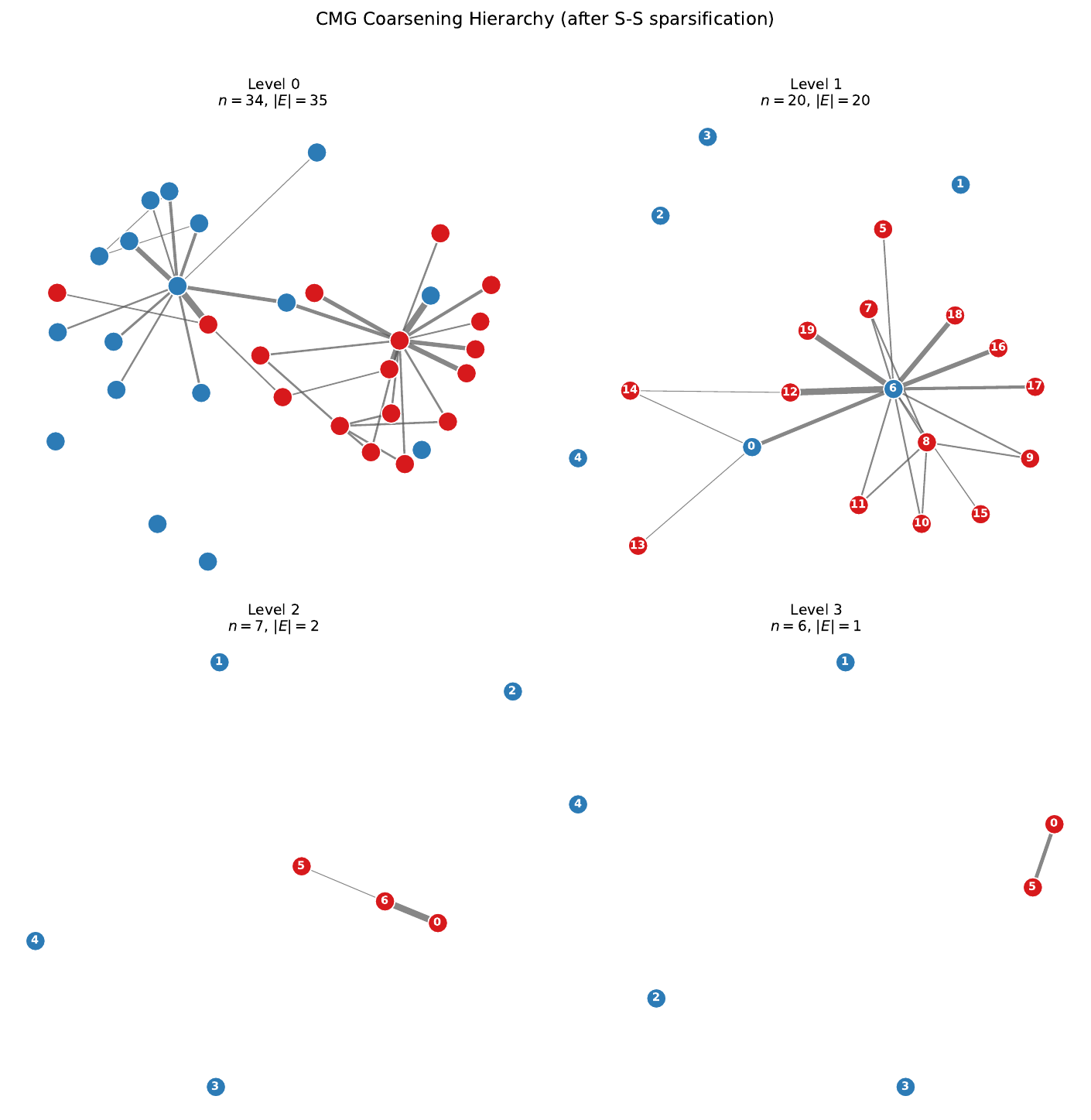}}\hfill
\subfigure[PPR, finest graph]{\includegraphics[width=0.46\textwidth]{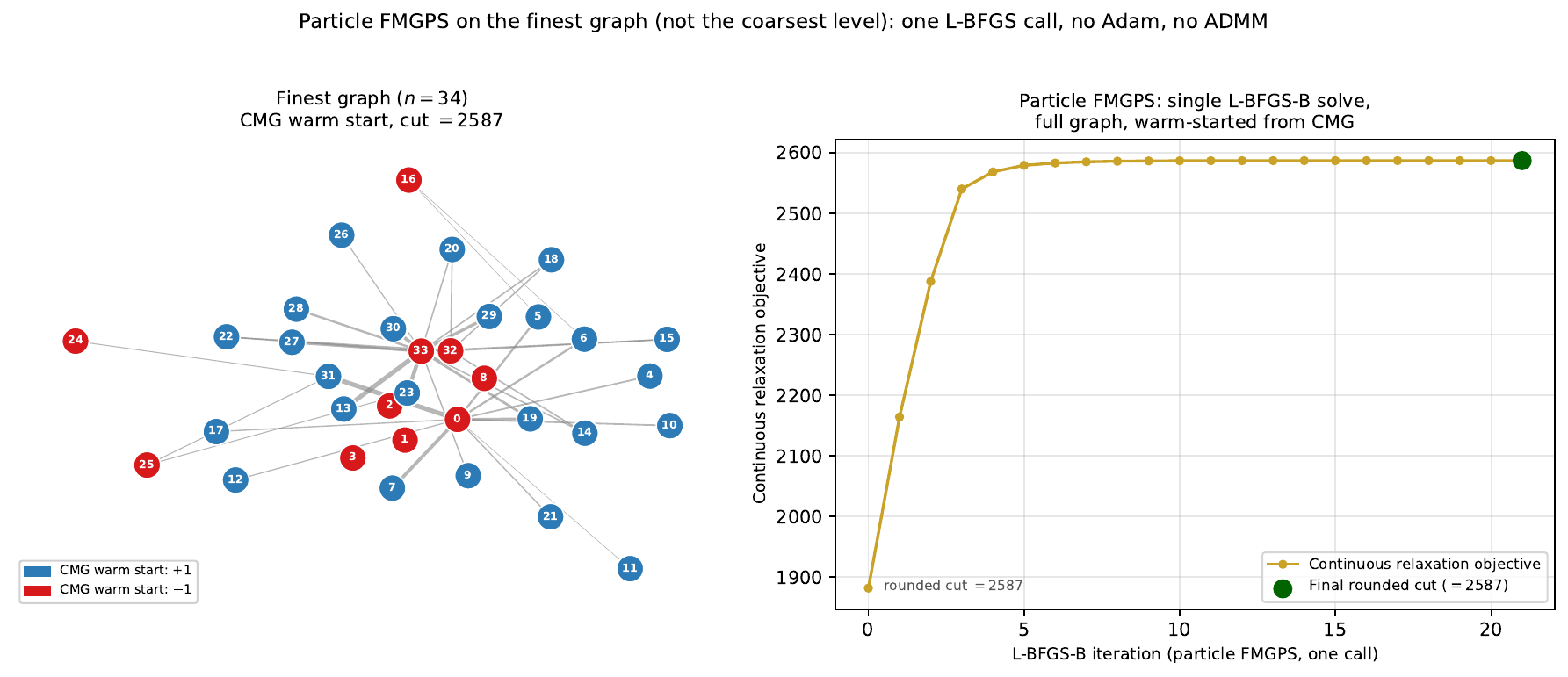}}
\caption{\small Pipeline walkthrough on Zachary's Karate Club. The four panels show
  the input graph, sparsified graph, CMG grouping, and a PPR solve on the
  finest graph initialized from the CMG projection.}
\label{fig:walkthrough}
\end{figure}

On this small instance, the CMG warm start already reaches the best cut found
by the later stages. A Burer-Monteiro~\citep{burer2003} relax-and-reround
pass gives a converged SDP relaxation \emph{estimate} of approximately
$2587.18$, matching
the best rounded cut to a gap of about $-7.7\times10^{-6}$, a negative gap,
which demonstrates directly that this estimate is not a certified upper
bound (a genuine upper bound cannot fall below a feasible cut); see
Table~\ref{tab:tail_bm}. This example is
therefore useful mainly as a transparent mechanical walkthrough, not as
evidence that every stage adds value on every graph.

\subsubsection{The continuous-phase optimizer swap in full}
\label{app:optimizer}

Table~\ref{tab:optimizer} gives the per-instance breakdown for the optimizer
swap summarized in Section~\ref{sec:results}. L-BFGS, Adam, GD, and a
level-stochastic variant use the same CMG warm start, restart budget,
continuation schedule, and downstream tail.

\begin{table}[htbp!]
\small\centering
\caption{\small Continuous-phase optimizer swap: light PPR solved by L-BFGS
  (production), Adam, GD, and level-stochastic SGD, each feeding the identical
  tail (2-opt$\to$KL$\to$3-opt$\to$$200$-restart$\to$memetic GA), shared CMG
  warm start, independently seeded per arm. Columns are cut after the
  continuous phase alone (pre-tail) and the final cut (post-tail), as a ratio
  to L-BFGS's own cut on that instance; last column is GD's continuous-phase
  wall clock relative to L-BFGS's. Two representative rows per regime shown;
  the max spread row summarizes all eight.}
\label{tab:optimizer}
\setlength{\tabcolsep}{4.5pt}
\begin{tabular}{lr rrr rrr r}\toprule
& & \multicolumn{3}{c}{Pre-tail (/L-BFGS)} & \multicolumn{3}{c}{Post-tail (/L-BFGS)} & \\
\cmidrule(lr){3-5}\cmidrule(lr){6-8}
Instance & $n$ & Adam & GD & SGD & Adam & GD & SGD & $t_{\mathrm{GD}}/t_{\mathrm{LB}}$\\\midrule
\multicolumn{9}{l}{\emph{Sparse-signed toroidal}}\\
\texttt{G33} & 2000 & 0.938 & 0.944 & 1.000 & 1.003 & 1.001 & 1.000 & $27.5\times$\\
\texttt{G34} & 2000 & 0.946 & 0.942 & 0.995 & 0.997 & 1.000 & 1.000 & $29.0\times$\\
\multicolumn{9}{l}{\emph{Dense uniform-weight}}\\
\texttt{G1}  & 800  & 0.987 & 0.979 & 0.979 & 1.000 & 1.000 & 1.000 & $11.8\times$\\
\texttt{G22} & 2000 & 0.985 & 0.981 & 0.983 & 0.997 & 0.997 & 0.999 & $15.1\times$\\\midrule
max spread (8 instances) & & \multicolumn{3}{c}{up to $7.4\%$} & \multicolumn{3}{c}{$\leq0.7\%$} & $11$--$30\times$\\
\bottomrule
\end{tabular}
\end{table}

\subsubsection{Internal optimality certificate for \texttt{G81}}
\label{app:g81cert}

The tabu-diversified \citep{glover1986tabu,glover1997tabu} tail reaches a cut of $13962$ on \texttt{G81}, which is
$99.3\%$ of the externally reported value $14060$~\citep{zick2026cosm}. As an
internal check, we solved mixed-integer repair problems with HiGHS
\citep{huangfu2018highs} on every connected residual core of size up to $200$
vertices, covering $40$ cores in total. No improving repair was found. This
certifies the reported cut on those residual cores, but it should not be read
as an independent proof that the global \texttt{G81} optimum is $14060$.

\clearpage
\section{Pseudocode for every pipeline component}
\label{app:component_algorithms}

Appendix~\ref{app:pipeline_end} names the components of the pipeline and shows
how they compose. This appendix gives pseudocode for each component in
isolation, transcribed from the reference implementation that produced every
numerical result reported here, so that the pseudocode and the measured
behaviour agree exactly. Two components already have boxes and are not
repeated: the continuous solver itself is
Algorithm~\ref{alg:fmgps_maxcut} (Appendix~\ref{app:fmgps_algorithm}), and CMG
coarsening is Algorithm~\ref{alg:cmg_coarsen} (Appendix~\ref{app:cmg_steps}).

\SetKwProg{Fn}{Procedure}{:}{}
\SetKwFunction{FSSSparsify}{SSSparsify}
\SetKwFunction{FSQPCoarse}{SQPCoarseSolve}
\SetKwFunction{FOnePass}{OnePass}
\SetKwFunction{FMultiPass}{MultiPass}
\SetKwFunction{FTwoOpt}{TwoOptSweep}
\SetKwFunction{FKLRefine}{KLRefine}
\SetKwFunction{FThreeOpt}{ThreeOptSweep}
\SetKwFunction{FLocalOpt}{LocalOptimize}
\SetKwFunction{FMemeticGA}{MemeticGA}
\SetKwFunction{FBurerMonteiro}{BurerMonteiro}
\SetKwFunction{FLocalFlip}{LocalFlip}

\subsection{Stage 1: Spielman-Srivastava sparsification}

Algorithm~\ref{alg:ss_sparsify} realizes the Spielman-Srivastava
leverage-score sampler~\citep{spielman2011} of
Appendix~\ref{app:leverage}. For $n\leq500$ it computes exact effective
resistances via the Laplacian pseudoinverse; for larger $n$ it falls back to
the degree-proxy approximation $R_{ij}\approx 1/d_i+1/d_j$. It is invoked only
when $|E|>3n$; on the already-sparse toroidal instances that dominate our
benchmarks it is a no-op.

\begin{algorithm2e}[htbp!]
\small
\SetAlgoLined
\DontPrintSemicolon
\caption{\small Spielman-Srivastava effective-resistance sparsification}
\label{alg:ss_sparsify}
\KwIn{$G=(V,E,w)$; target edge count $t$}
\KwOut{sparsified $G'=(V,E',w')$}
\Fn{\FSSSparsify{$G$, $t$}}{
  \lIf{$|E|\leq t$}{\Return $G$ \tcp*{no-op}}
  \eIf{$n\leq 500$}{
    build Laplacian $L$ from $|w|$; eigendecompose $L=Q\Lambda Q^\top$\;
    $R_{ij}\gets\sum_{k:\lambda_k>\mathrm{tol}} (Q_{ik}-Q_{jk})^2/\lambda_k$ for
      every edge $(i,j)$\tcp*{exact resistances}
  }{
    $d_i\gets\sum_{j\sim i}|w_{ij}|$ for every vertex\;
    $R_{ij}\gets 1/d_i + 1/d_j$ for every edge $(i,j)$\tcp*{degree-proxy fallback}
  }
  $\ell_e \gets |w_e|\,R_e$ for every edge $e$\tcp*{leverage score}
  \ForEach{edge $e=(i,j)$}{
    $p_e \gets \min\!\big(1,\; \ell_e\, t / \textstyle\sum_{e'}\ell_{e'}\big)$\;
    with probability $p_e$: keep $e$ with rescaled weight
      $\mathrm{sign}(w_e)\,|w_e|/p_e$\;
  }
  \Return sparsified $G'=(V,E',w')$\;
}
\end{algorithm2e}

\subsection{Stage 3: coarse sweet-spot SQP solve}

Algorithm~\ref{alg:sqp_coarse} is the Burer-Monteiro~\citep{burer2003}
moment-matrix solve applied at the bottom of the CMG hierarchy, and is used
only at a level with $n_\ell\leq20$ vertices, where a small dense
SLSQP~\citep{kraft1988sqp} solve is affordable. Note that
here the binary constraint enters as an explicit equality on the moments of
Eq.~\eqref{eq:moment}  ($\mu^{(4)}-2\mu^{(2)}+\mu^{(0)}=0$, i.e.\
$\Ex[(x^2-1)^2]=0$, which forces each marginal onto $\{\pm1\}$), rather than
through the penalty schedule the larger levels use. Throughout,
$\mu_i^{(k),l}$ denotes the order-$k$ moment entry read off the Gram matrix
$R_i^{(l)\top}R_i^{(l)}$ of the rank-restricted factor.

\begin{algorithm2e}[htbp!]
\small
\SetAlgoLined
\DontPrintSemicolon
\caption{\small Coarse sweet-spot SQP solve ($n_\ell\leq20$)}
\label{alg:sqp_coarse}
\KwIn{coarse graph $G_\ell$; CMG projection $x_0$}
\KwOut{best $\pm1$ partition on $G_\ell$}
\Fn{\FSQPCoarse{$G_\ell$, $x_0$}}{
  $\mathrm{best}\gets-\infty$\;
  \For{$x_{\mathrm{init}} \in \{x_0,\ \text{one random } \pm1\ \text{draw}\}$}{
    initialize BM factors $R_i^{(l)}\in\mathbb{R}^{\mathrm{rank}\times3}$ at the
      point mass $x_{\mathrm{init}}$\;
    solve via SLSQP:
      $\displaystyle\min -\tfrac14\sum_{(i,j)\in E} w_{ij}\big(1-\textstyle\sum_l \mu_i^{(1),l}\mu_j^{(1),l}\big)$\;
    \phantom{solve via SLSQP:}
      s.t.\ $\sum_l \mu_i^{(0),l}=1$,\
      $\sum_l(\mu_i^{(4),l}-2\mu_i^{(2),l}+\mu_i^{(0),l})=0$\ \ $\forall i$\;
    $\hat{x}_i \gets \mathrm{sign}\big(\sum_l \mu_i^{(1),l}\big)$ for all $i$\;
    \FLocalFlip{$\hat{x}$, $G_\ell$}\;
    keep $\hat{x}$ if $\mathrm{Cut}(\hat{x}) > \mathrm{best}$\;
  }
  \Return the best $\hat{x}$ found across the two restarts\;
}
\end{algorithm2e}

\subsection{Stage 4: one particle pass, multi-pass, and 2-opt}

Algorithm~\ref{alg:one_pass} is Algorithm~\ref{alg:fmgps_maxcut} restricted to
a fixed $L,R,k$ and packaged for repeated warm-started invocation; it is what
Algorithm~\ref{alg:multi_pass} chains across passes. We keep it as a separate
box because the multi-pass driver of Appendix~\ref{app:branches} calls it
directly rather than through the full solver. Algorithm~\ref{alg:two_opt} is
the 2-opt~\citep{croes1958tsp} closing polish whose \emph{ordering} matters: applied once at the end it is
worth $+0.023$ in mean ratio, while applying it after every pass is worse
(Appendix~\ref{app:t2g20_ladder}).

\begin{algorithm2e}[htbp!]
\small
\SetAlgoLined
\DontPrintSemicolon
\caption{\small One particle pass}
\label{alg:one_pass}
\KwIn{$J$; warm start $x_{\mathrm{warm}}$; $L,R,k$; sat; $\rho_1<\cdots<\rho_S$}
\KwOut{best cut found and the corresponding partition $\hat{x}$}
\Fn{\FOnePass{$J$, $x_{\mathrm{warm}}$, $L,R,k$, sat, $\rho_1<\cdots<\rho_S$}}{
  (identical body to Algorithm~\ref{alg:fmgps_maxcut}'s restart loop,
    warm-started from $x_{\mathrm{warm}}$)\;
  \Return best cut found, and the corresponding partition $\hat{x}$\;
}
\end{algorithm2e}

\begin{algorithm2e}[htbp!]
\small
\SetAlgoLined
\DontPrintSemicolon
\caption{\small Multi-pass restarts}
\label{alg:multi_pass}
\KwIn{$J$; initial partition $x_0$; edges; weights; number of passes}
\KwOut{best cut and its partition}
\Fn{\FMultiPass{$J$, $x_0$, edges, weights, passes}}{
  $x_{\mathrm{warm}} \gets x_0$;\ $\mathrm{best}\gets-\infty$\;
  \For{$p = 1,\ldots,\mathrm{passes}$}{
    $(c,\hat{x}) \gets$ \FOnePass{$J$, $x_{\mathrm{warm}}$, \ldots}\;
    \If{$c > \mathrm{best}$}{
      $\mathrm{best}\gets c$;\ $x_{\mathrm{warm}}\gets\hat{x}$
        \tcp*{chain the best binary forward}
    }
  }
  \Return $\mathrm{best}$, $x_{\mathrm{warm}}$\;
}
\end{algorithm2e}

\begin{algorithm2e}[htbp!]
\small
\SetAlgoLined
\DontPrintSemicolon
\caption{\small 2-opt sweep}
\label{alg:two_opt}
\KwIn{partition $x$; $G=(V,E,w)$}
\KwOut{number of accepted double-flips}
\Fn{\FTwoOpt{$x$, $G$}}{
  $g_i \gets \sum_{j\sim i} w_{ij}x_j$ for all $i$\;
  \Repeat{not $\mathrm{improved}$}{
    $\mathrm{improved}\gets\mathrm{false}$\;
    \For{$i = 1,\ldots,n$}{
      \For{$j = i+1,\ldots,n$}{
        $\Delta \gets 2x_ig_i + 2x_jg_j - 4w_{ij}x_ix_j$\;
        \If{$\Delta > 10^{-10}$}{
          update $g$ for neighbors of $i$ and of $j$; flip
            $x_i\gets-x_i$, $x_j\gets-x_j$\;
          $\mathrm{improved}\gets\mathrm{true}$\;
        }
      }
    }
  }
  \Return number of accepted double-flips\;
}
\end{algorithm2e}

\subsection{Stage 5 (optional discrete tail): Kernighan-Lin and 3-opt}

Algorithms~\ref{alg:kl_refine} (Kernighan-Lin~\citep{kernighan1970efficient})
and~\ref{alg:three_opt} (3-opt~\citep{lin1973heuristic}) are the first two tail
stages ablated in Appendix~\ref{app:tail_anatomy}; both are monotone (neither
ever decreases the cut) and neither is part of the headline light pipeline. KL
is the stage that does the bulk of the tail's work, which is why the endpoint is
set by the search operator rather than by the starting partition.

\begin{algorithm2e}[htbp!]
\small
\SetAlgoLined
\DontPrintSemicolon
\caption{\small Kernighan-Lin variable-depth exchange}
\label{alg:kl_refine}
\KwIn{partition $x$; $G$}
\KwOut{total cut improvement}
\Fn{\FKLRefine{$x$, $G$}}{
  $\mathrm{total}\gets0$\;
  \Repeat{$\mathrm{best\_cum} \leq 10^{-9}$}{
    $g_i \gets \sum_{j\sim i} w_{ij}x_j$ for all $i$; unlock all vertices\;
    $\mathrm{cum}\gets0$;\ $\mathrm{best\_cum}\gets0$;\
      $\mathrm{best\_k}\gets0$;\ $\mathrm{moves}\gets[\,]$\;
    \For{step $=1,\ldots,n$}{
      $v \gets \arg\max_{v\ \mathrm{unlocked}} x_vg_v$;
        \textbf{break} if none unlocked\;
      $\mathrm{cum}\gets\mathrm{cum}+x_vg_v$; update $g$ for neighbors of $v$
        under $x_v\to-x_v$\;
      flip $x_v\gets-x_v$; lock $v$; append $v$ to $\mathrm{moves}$\;
      \If{$\mathrm{cum} > \mathrm{best\_cum}$}{
        $\mathrm{best\_cum}\gets\mathrm{cum}$;\
          $\mathrm{best\_k}\gets|\mathrm{moves}|$\;
      }
    }
    undo every move in $\mathrm{moves}$ after index $\mathrm{best\_k}$
      \tcp*{keep only the best prefix}
    \lIf{$\mathrm{best\_cum} > 10^{-9}$}{$\mathrm{total}\gets\mathrm{total}+\mathrm{best\_cum}$}
  }
  \Return $\mathrm{total}$\;
}
\end{algorithm2e}

\begin{algorithm2e}[htbp!]
\small
\SetAlgoLined
\DontPrintSemicolon
\caption{\small Connected-triple 3-opt}
\label{alg:three_opt}
\KwIn{partition $x$; $G$}
\KwOut{total cut improvement}
\Fn{\FThreeOpt{$x$, $G$}}{
  $g_i \gets \sum_{j\sim i} w_{ij}x_j$ for all $i$;\ $\mathrm{total}\gets0$\;
  \Repeat{not $\mathrm{improved}$}{
    $\mathrm{improved}\gets\mathrm{false}$\;
    \ForEach{edge $(a,b)$, $a<b$}{
      $\Delta_{ab} \gets x_ag_a + x_bg_b - 2w_{ab}x_ax_b$
        \tcp*{pair-flip $\Delta\mathrm{Cut}$}
      \If{$\Delta_{ab} > 10^{-9}$}{
        flip $\{a,b\}$, update $g$;
          $\mathrm{total}\gets\mathrm{total}+\Delta_{ab}$;
          $\mathrm{improved}\gets\mathrm{true}$; \textbf{continue}\;
      }
      \ForEach{$c$ adjacent to $a$ or $b$, $c\notin\{a,b\}$}{
        $S\gets\{a,b,c\}$;\
          $\Delta_S \gets \sum_{v\in S}x_vg_v - 2\sum_{(u,v)\in E\cap S}w_{uv}x_ux_v$\;
        \If{$\Delta_S > 10^{-9}$}{
          flip $S$, update $g$; $\mathrm{total}\gets\mathrm{total}+\Delta_S$;
            $\mathrm{improved}\gets\mathrm{true}$; \textbf{break}\;
        }
      }
    }
  }
  \Return $\mathrm{total}$\;
}
\end{algorithm2e}

\subsection{Stage 5 (optional discrete tail): local optimum, memetic GA, and Burer-Monteiro}

Algorithm~\ref{alg:local_optimize} drives a partition to a joint
2-opt~\citep{croes1958tsp}/Kernighan-Lin~\citep{kernighan1970efficient}/3-opt~\citep{lin1973heuristic}
local optimum; it is the polish step used both by the
multi-restart diagnostic and, inside Algorithm~\ref{alg:memetic_ga}, to
evaluate every population member and every child.
Algorithm~\ref{alg:memetic_ga} is the memetic genetic
algorithm~\citep{moscato2004memetic}, using the Chu-Beasley~\citep{chu1998genetic}
steady-state replacement scheme, whose
\texttt{ga\_only} control appears in Appendix~\ref{app:tail_anatomy}, and
Algorithm~\ref{alg:burer_monteiro} is the Burer-Monteiro~\citep{burer2003}
relax-and-reround pass, computed via the Mixing
Method~\citep{wang2017mixing}, that
produces the empirical SDP-relaxation estimate of Table~\ref{tab:tail_bm}.
Its default rank
$\lceil\sqrt{2n}\rceil$ is the threshold above which the rank-restricted
problem is \emph{generically} free of spurious local
minima~\citep{boumal2016burermonteiro}, but the Mixing Method's convergence
check (a tolerance on successive objective values) does not itself verify
that condition, so the resulting estimate is not a certified bound; 
Appendix~\ref{app:walkthrough}'s Karate Club walkthrough gives a direct,
concrete negative-gap counterexample.

\begin{algorithm2e}[htbp!]
\small
\SetAlgoLined
\DontPrintSemicolon
\caption{\small Joint 2-opt/KL/3-opt local optimum}
\label{alg:local_optimize}
\KwIn{partition $x$; $G$}
\KwOut{$x$ at a joint 2-opt/KL/3-opt local optimum (modified in place)}
\Fn{\FLocalOpt{$x$, $G$}}{
  \FTwoOpt{$x$, $G$}\;
  \For{$\mathrm{it} =1,\ldots,50$}{
    $a \gets$ \FKLRefine{$x$, $G$};\ $b \gets$ \FThreeOpt{$x$, $G$}\;
    \lIf{$a+b < 10^{-9}$}{\textbf{break}}
  }
}
\end{algorithm2e}

\begin{algorithm2e}[htbp!]
\small
\SetAlgoLined
\DontPrintSemicolon
\caption{\small Memetic GA (Chu-Beasley steady-state, agreement crossover)}
\label{alg:memetic_ga}
\KwIn{$G$; population size $\mathrm{pop}$; generations $\mathrm{gens}$}
\KwOut{fittest population member and its cut}
\Fn{\FMemeticGA{$G$, $\mathrm{pop}$, $\mathrm{gens}$}}{
  \For{$i = 1,\ldots,\mathrm{pop}$}{
    $P_i \gets$ random $\pm1$; \FLocalOpt{$P_i$, $G$};
      $\mathrm{fit}_i\gets\mathrm{Cut}(P_i)$\;
  }
  \For{$\mathrm{gen} = 1,\ldots,\mathrm{gens}$}{
    $p_1,p_2 \gets$ two distinct binary-tournament winners from $P$\;
    align $P_{p_2}$ to $P_{p_1}$ by global sign (maximize agreement)\;
    $\mathrm{child}_i \gets P_{p_1,i}$ where they agree; a random $\pm1$ coin
      flip where they disagree\;
    \FLocalOpt{$\mathrm{child}$, $G$}; $c_f\gets\mathrm{Cut}(\mathrm{child})$\;
    \If{no population member has $|\mathrm{fit}-c_f|<0.5$}{
      \tcp{reject duplicate-cut children}
      $w \gets \arg\min_i \mathrm{fit}_i$\;
      \If{$c_f > \mathrm{fit}_w$}{
        $P_w \gets \mathrm{child}$;\ $\mathrm{fit}_w \gets c_f$
          \tcp*{steady-state replace-worst}
      }
    }
  }
  \Return $\arg\max_i \mathrm{fit}_i$ and its cut\;
}
\end{algorithm2e}

\begin{algorithm2e}[htbp!]
\small
\SetAlgoLined
\DontPrintSemicolon
\caption{\small Burer-Monteiro relax-and-reround (Mixing Method)}
\label{alg:burer_monteiro}
\KwIn{$G$; rank $p$; number of hyperplanes $h$}
\KwOut{$(\mathrm{bound},\mathrm{best})$, whose difference is an empirical,
  not certified, optimality gap}
\Fn{\FBurerMonteiro{$G$, $p$, $h$}}{
  \lIf{$p \leq 0$}{$p \gets \lceil\sqrt{2n}\rceil$\tcp*{default rank}}
  $V_i \gets$ random unit vector in $\mathbb{R}^p$ for all $i$\;
  \Repeat{$|\mathrm{obj}_{\mathrm{prev}}-\mathrm{obj}| < \mathrm{tol}\cdot(1+|\mathrm{obj}|)$}{
    \For{$i=1,\ldots,n$}{
      $g_i \gets \sum_{j\sim i} w_{ij}V_j$;\ $V_i \gets -g_i/\lVert g_i\rVert$
        \tcp*{anti-align (coordinate update)}
    }
    $\mathrm{obj} \gets \tfrac12\big(W - \sum_{(i,j)\in E} w_{ij}\langle V_i,V_j\rangle\big)$\;
  }
  $\mathrm{bound}\gets\mathrm{obj}$;\ $\mathrm{best}\gets-\infty$\;
  \For{$h' = 1,\ldots,h$}{
    draw random hyperplane $r\sim\mathcal{N}(0,I_p)$;
      $x_i\gets\mathrm{sign}(\langle V_i,r\rangle)$ for all $i$\;
    \FLocalOpt{$x$, $G$}; keep $x$ if $\mathrm{Cut}(x)>\mathrm{best}$\;
  }
  \Return $(\mathrm{bound}, \mathrm{best})$
    \tcp*{$\mathrm{bound}-\mathrm{best}$ is an empirical, not certified, gap}
}
\end{algorithm2e}

\end{document}